\documentclass[12pt, final]{l4dc2023}

\title[Uncertainty-Aware Implicit Safe Set Algorithm]{Probabilistic Safeguard for Reinforcement Learning\\Using Safety Index Guided Gaussian Process Models}
\usepackage{times}
\usepackage{cleveref}
\usepackage{algorithm}
\usepackage{comment}
\usepackage[noend]{algpseudocode}
\usepackage{booktabs}
\usepackage{multirow}
\usepackage{bm}
\usepackage{bbm}
\usepackage{bbding}
\usepackage{wrapfig}
\usepackage[font = footnotesize]{caption}
\newtheorem{asm}{Assumption}

\Crefname{asm}{Assumption}{Assumption}
\Crefname{definition}{Definition}{Definition}
\Crefname{lemma}{Lemma}{Lemma}

\DeclareMathOperator*{\argmax}{arg\,max}
\DeclareMathOperator*{\argmin}{arg\,min}
\DeclareMathOperator{\Tr}{Tr}

\author{\Name{Weiye Zhao$^1$} \Email{weiyezha@andrew.cmu.edu} \\
 \Name{Tairan He$^1$} \Email{tairanh@andrew.cmu.edu}\\
 \Name{Changliu Liu} \Email{cliu6@andrew.cmu.edu}\\
 \addr Robotics Institute, Carnegie Mellon University, Pittsburgh, PA 15213 USA 
  \thanks{This material is based upon work supported by the National Science Foundation under Grant No. 2144489.}
 }

\begin{document}

\maketitle
\def\thefootnote{1}\footnotetext{These authors contributed equally to this work.}\def\thefootnote{\arabic{footnote}}

\begin{abstract}%
Safety is one of the biggest concerns to applying reinforcement learning (RL) to the physical world. In its core part, it is challenging to ensure RL agents persistently satisfy a hard state constraint without white-box or black-box dynamics models. This paper presents an integrated model learning and safe control framework to safeguard any RL agent, where the environment dynamics are learned as Gaussian processes. The proposed theory provides (i) a novel method to construct an offline dataset for model learning that best achieves safety requirements; (ii) a design rule to construct the safety index to ensure the existence of safe control under control limits; (iii) a probablistic safety guarantee (i.e. probabilistic forward invariance) when the model is learned using the aforementioned dataset. Simulation results show that our framework achieves almost zero safety violation on various continuous control tasks.
\end{abstract}


\section{Introduction}
While reinforcement learning (RL) has achieved impressive results in games like Atari~\citep{zhao2019stochastic}, Go~\citep{silver2017masteringgo} and Starcraft~\citep{vinyals2019grandmaster}, the lack of safety guarantee limits the application of RL algorithms on real-world physical systems such as robotics~\citep{wei2022persistently}. In its core part, it is critical to ensure that RL agents persistently satisfy a hard state constraint defined by a \textit{safe set} (e.g., a set of non-colliding states) in many robotic applications~\citep{zhao2021issa, zhao2020contact,zhao2020experimental}. 
Though various constrained RL algorithms~\citep{he2023autocost,achiam2017cpo,wachi2018safe,yang2021wcsac,zhao2023state} have been introduced, the trial-and-error mechanism of these methods makes it hard to avoid safety violations during policy learning.

On the other hand, when the dynamics model of the system is accessible, energy-function-based safe control methods can achieve the safety guarantee, i.e., persistently satisfying the hard state constraint. These methods~\citep{noren2021safe,zhao2022safety,liu2014control,gracia2013reactive,he2023hierarchical} first synthesize an energy function such that the safe states have low energy, and then design a control law to satisfy the safe action constraints, i.e., to dissipate energy. Then these methods ensure \textit{forward invariance} inside the safe set 
. However, their limitation is that they exploit either white-box dynamics models (e.g., analytic form) ~\citep{khatib1986real,ames2014control,liu2014control,gracia2013reactive} or black-box dynamics models (e.g., digital twin simulators)~\citep{zhao2021issa}, while these models are not easy to build in complex environments. Other related works are summarized in \Cref{sec: related work}.

Practically, compared to dynamics models (i.e., a full mapping from the current state and control to the next state), it is easier to obtain samples of the dynamic transitions in real-world applications \citep{huang2018apolloscape,caesar2020nuscenes,cheng2019human,sun2023hybrid}. This paper investigates approaches to utilize these transition samples to achieve safety guarantees under the energy-function-based safe control framework, while relaxing the requirements of white-box or black-box dynamics models. In our methods, we leverage Gaussian Process (GP) to learn a statistical dynamics model due to (i) GP's reliable estimate of uncertainty~\citep{williams2006gaussian}; (ii) its well-established theory on uniform error bounds~\citep{srinivas2009gaussian,srinivas2012information,chowdhury2017kernelized,kanagawa2018gaussian,lederer2019uniform}. Instead of performing online model learning using online data, our dynamics model is learned based on an offline constructed dataset. When the dataset is constructed offline, we have the full control over the data distribution, which could result in (i) reliable convergence in model learning and (ii) good safety guarantees. 

\begin{figure}
    \centering
    \includegraphics[width=0.7\columnwidth]{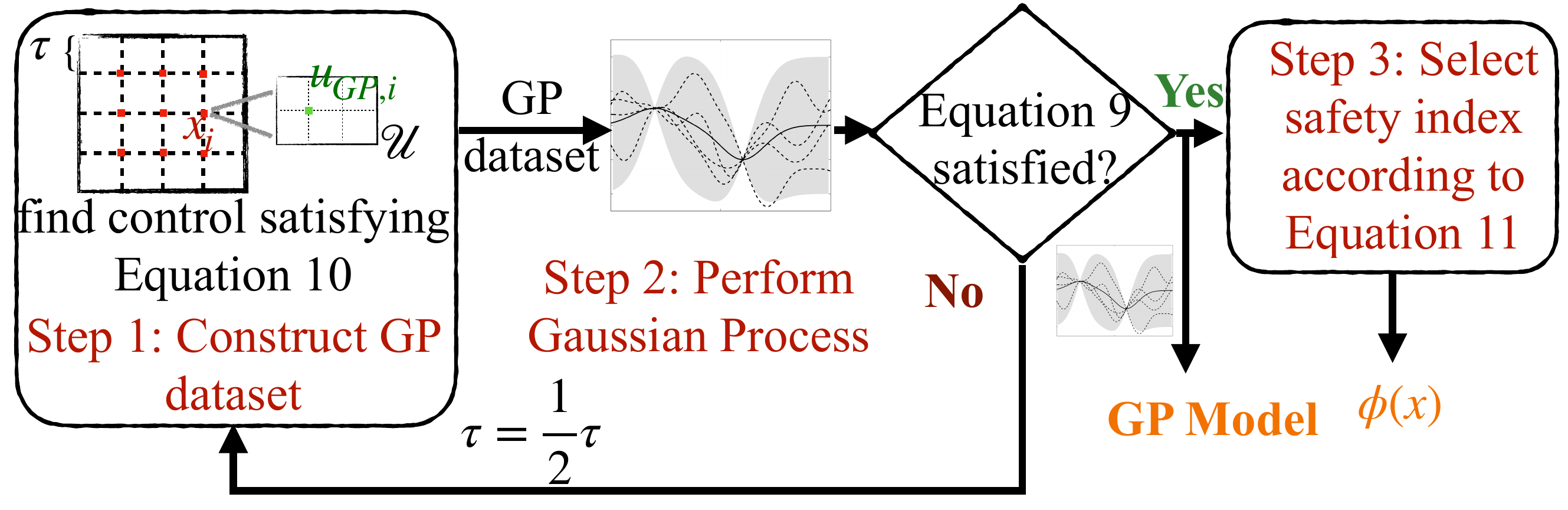}
    \caption{The flow chart that illustrate the offline state of UAISSA. We first select a proper state space discretization step size, construct the offline dynamics learning dataset for GP and design parameters of the safety index.}
    \label{fig:uaissa_offline}
    \vspace{-10pt}
\end{figure}

The main contribution of this paper is a theory to probabilistically safeguard robot policy learning using energy-function-based safe control with a GP dynamics model learned on an offline dataset.  
The overall pipeline of our method is shown in \Cref{fig:uaissa_offline}.
To achieve our goal of safeguarding RL agents with GP dynamics model, we first show how to construct the dataset for model learning and how to design the associated energy function (called \textit{safety index}) so that there always exists a feasible safe control under control limits. 
Secondly, 
we show how to design a safeguard for arbitrary RL agents to guarantee forward invariance during policy learning.
The method is evaluated on various challenging continuous control problems where the RL agents achieve almost zero constraint violation during policy learning. Additional results and discussions can be found in the appendix of the arxiv version \url{https://arxiv.org/abs/2210.01041}.

\section{Problem Background}
\subsection{Notations}
\paragraph{Dynamics}
 Denote $x_t \in \mathcal{X}\subset \mathbb{R}^{n_x}$ as the robot state at time step $t$; $u_t \in \mathcal{U}\subset \mathbb{R}^{n_u}$ as the control input to the robot at time step $t$, and the control space $\mathcal{U}$ is bounded. And denote $\mathcal{W}:=\mathcal{X}\times\mathcal{U}$, which is assumed to be compact. The system dynamics are defined as: 
\begin{equation}\label{eq:dynamics fn}
\begin{split}
    &x_{t+1} = f(x_t,u_t), \\
\end{split}
\end{equation}
where $f: \mathcal{W} \rightarrow \mathcal{X}$ is a function that maps the current robot state and control to the robot state in the next time step, and $f( \cdot )$ is $L_f$ Lipschitz continuous with respect to the 1-norm. 
For simplicity, this paper considers deterministic dynamics (but is unknown).

\paragraph{Safety Specification}
The safety specification requires that the system state should be constrained in a closed subset in the state space, called the safe set $\mathcal{X}_S$. The safe set can be represented by the zero-sublevel set of a continuous and piecewise smooth function $\phi_0:\mathbb{R}^{n_x} \rightarrow \mathbb{R}$, i.e., $\mathcal{X}_S = \{x\mid \phi_0(x)\leq 0\}$. $\mathcal{X}_S$ and $\phi_0$ are directly specified by users. 

\subsection{Preliminary}
\paragraph{Gaussian Process}
A Gaussian process (GP)~\citep{williams2006gaussian} is a nonparametric regression method specified by its mean $\mu_g(z) = \mathbb{E}[g(z)]$ and covariance (kernel) functions $k(z,z')=\mathbb{E}[(g(z)-\mu(z))(g(z') - \mu(z'))]$. 
Given $N$ finite measurements \\$y_N = [y(z_1), y(z_2), \cdots, y(z_N)]^T$ of the unknown function $g: \mathbb{R}^D \to \mathbb{R}$ subject to independent Gaussian noise $v \sim \mathcal{N}(0, \sigma_{\text{noise}}^2)$, the posterior mean $\mu(z_*)$ and variance $\sigma^2(z_*)$ are calculated as:
\begin{align}
    \label{eq:gp_mu_sigma}
    & \mu(z_*) = k^T_*(z_*)(K+ \sigma_{\text{noise}}^2 I_N)^{-1}y_N \\
    & \sigma^2(z_*) = k(z_*, z_*) - k^T_*(z_*)(K + \sigma_{\text{noise}}^2 I_N)^{-1}k_*(z_*) \label{eq: sigma gp},
\end{align}
where $K_{i,j} = k(z_i, z_j)$ and $k_*(z_*) = [k(z_1, z_*), k(z_2, z_*), \cdots, k(z_N, z_*)]^T$. In the following discussions, we assume observation is noise-free, i.e. $\sigma_{\text{noise}} = 0$. Note that noise reduction methods can be applied to eliminate $\sigma_{\text{noise}}$ in practice~\citep{kostelich1993noise}. For most commonly used kernel functions, GP can approximate any continuous function on any compact subset of $\mathcal{Z}$~\citep{srinivas2012information}. In this paper, the dynamics are modeled using GP with the following definition.

\begin{definition}[GP Dynamics Model]
\label{definition:gaussian_process}
The dynamics model of $f$ in \eqref{eq:dynamics fn} is represented as a zero mean Gaussian process
with
a continuous covariance kernel $k(\cdot, \cdot)$ with Lipschitz constant $L_k$ on the compact set $\mathcal{W}$, where $L_k$ can be caluclated analytically for commonly-used kernels~\citep{lederer2019uniform}. The posterior mean function and covariance matrix function of the GP model are denoted as $\mu_f(\cdot)$ and $\Sigma_f(\cdot)$, respectively. 
\end{definition}


\paragraph{Safety Index} To ensure system safety, 
all visited states should be inside $\mathcal{X}_S$
. However, $\mathcal{X}_S$ 
may contain states that will inevitably go to the unsafe set no matter what control inputs we choose. Hence, we need to assign high energy values to those inevitably unsafe states, and ensure \textbf{forward invariance} in a subset of the safe set $\mathcal{X}_S$. Safe Set Algorithm (SSA)~\citep{liu2014control}
synthesizes the energy function as a continuous, piece-wise smooth scalar function
$\phi:\mathbb{R}^{n_x} \rightarrow \mathbb{R}$, named, safety index. And we denote its 0-sublevel set as $\mathcal{X}_S^D:= \{x|\phi(x) \leq 0\}$. The general form of the safety index was proposed as $\phi = \phi_0^* + k_1\dot\phi_0 + \cdots + k_n\phi_0^{(n)}$ where (i) the roots of $1+k_1s+\ldots + k_n s^n=0$ are all negative real (to ensure zero-overshooting of the original safety constraints); (ii) the relative degree from $\phi_0^{(n)}$ to $u$ is one (to avoid singularity); and (iii) $\phi_0^*$ defines the same zero sublevel set as $\phi_0$ (to nonlinear shape the gradient of $\phi$ at the boundary of the safe set). It is shown in \citep{liu2014control} that choosing a control that decreases $\phi$ whenever $\phi$ is greater than or equal to $0$ can ensure forward invariance inside $\mathcal{X}_S\cap\mathcal{X}_S^D$. 


\subsection{Problem Formulation}
\label{sec: prob form}

The core problem of this paper is to safeguard a nominal controller (i.e., an RL agent) such that all visited states are inside $\mathcal{X}_S$. In this paper,  we are specifically interested in \textit{degree two systems} (i.e., the relative degree from $\dot\phi_0$ to $u$ is one), and the safety specification is defined as ${\phi_0} = d_{min} - d$ where $d$ denotes the safety status of the system, and the system becomes more unsafe when $d$ decreases. For example, for collision avoidance, $d$ can be designed to measure the relative distance between the robot and obstacles, which needs to be greater than some threshold $d_{min}$. Following the rules in~\citep{liu2014control}, we parameterize the safety index as 
$\phi = \sigma + d_{min}^n - d^n - k\dot d$,
and $\sigma, n, k > 0$ are tunable parameters of the safety index. It is easy to verify that this design satisfies the three requirements discussed above. 

The nominal control is an RL controller which aims to maximize cumulative discounted rewards in 
an infinite-horizon deterministic Markov decision process
(MDP). An MDP is specified by a tuple $(\mathcal{X}, \mathcal{U}, \gamma, r, f)$, where $r: \mathcal{X} \times \mathcal{U} \rightarrow \mathbb{R}$ is the reward function, $ 0 \leq \gamma < 1$ is the discount factor, and $f$ is the deterministic system dynamics defined in \eqref{eq:dynamics fn}, and we can access data samples of $f$.
We then define the 
set of safe control as $\mathcal{U}_S^D(x):=\{u\in \mathcal{U}\mid \phi(f(x, u)) \leq \max\{\phi(x)-\eta, 0\} \}$, where $\eta$ is a positive constant. Hence, the nominal controller can be safeguarded by projecting the nominal control $u_t^r$ to $\mathcal{U}_S^D(x)$ by solving the following optimization: 
\begin{equation}\label{eq:safeguard_optimization}
\begin{split}
    &\min_{u_t\in\mathcal{U}} \| u_t - u^r_t\|^2\\
    & \text{s.t. } \phi(f(x_t, u_t)) \leq \max\{\phi(x_t)-\eta, 0\}.
\end{split}
\end{equation}

Since $f$ is unknown, we need to first learn a statistical model of $f$ and then solve \eqref{eq:safeguard_optimization}. The well-established theories on uniform error bounds~\citep{srinivas2009gaussian,srinivas2012information,chowdhury2017kernelized,kanagawa2018gaussian,lederer2019uniform} for GP allows us to a build a reliable statistical model for a given dataset.
\begin{lemma}[Well-Calibrated Model]
For a dataset and $\delta \in (0,1)$, there exists $\beta_f(\delta)$ that we can learn a GP model $\big\{\mu_f(x, u), \sigma_f(x,u)\big\}$ that satisfies: $\forall x \in \mathcal{X}, u \in \mathcal{U}, P\Big(||f(x, u) - \mu_f(x, u)||_1 \leq \beta_f \sigma_f(x, u)\Big) \geq 1 - \delta$, where $\beta_f$ means $\beta_f(\delta)$ for simplicity and $\sigma_f(x,u) = \Tr(\Sigma_f^{\frac{1}{2}}(x,u))$.
\label{lem:Well-calibrated}
\end{lemma}
This lemma ensures that the confidence intervals of GP prediction cover the true dynamics function with high probability given an appropriate constant $\beta_f$. The expressions of $\beta_f$ are discussed in~\citep{srinivas2009gaussian,srinivas2012information,chowdhury2017kernelized,kanagawa2018gaussian,lederer2019uniform}.


\paragraph{Challenges}
The challenges for solving \eqref{eq:safeguard_optimization} can be divided into two parts: (1) \textbf{offline synthesis stage}: how to generate a data set for model learning and safety index synthesis such that: (a) there is always a solution for \eqref{eq:safeguard_optimization} with the learnt dynamics under control limit; (b) safety is preserved under model mismatch.
(2) \textbf{online computation stage}: how to efficiently solve \eqref{eq:safeguard_optimization} with learnt dynamics to find safe controls.

\section{Offline Safety Index Synthesis}
In this section, we introduce the theoretical results to tackle the aforementioned offline synthesis stage challenges. We first show that the deterministic constraint \eqref{eq:safeguard_optimization} can be verified via introducing an upper bound of safety index. Next, we introduce a theory that verifies the feasibility of  the probabilistic constraint for all possible states (which is uncountably many) by verifying the feasibility of a similar problem for finitely many states (Proposition \ref{proposition:safe_control_discretization}). Lastly, we discuss the criteria of dataset construction for model learning and the associated safety index design rule which ensure nonempty set of safe control for all possible system states (\Cref{theo: nonempty set of safe control}). 

\subsection{Preserving Safety with Learnt Model}
As mentioned in \Cref{sec: prob form}, our ultimate goal is to solve \eqref{eq:safeguard_optimization}, whereas it is intractable to directly solve \eqref{eq:safeguard_optimization} since $f$ is unknown. On the other hand, we can learn a reliable statistical model of $f$ via GP, i.e. $\big\{\mu_f(x,u), \sigma_f(x,u), \beta_f\big\}$ as stated in Lemma \ref{lem:Well-calibrated}. Hence, as long as we can find a probabilistic upper bound of safety index $\phi(f(x,u))$, denoted as $\mathbf{U}_f(x,u)$ such that $\mathbf{U}_f(x,u) \geq \phi(f(x,u))$, the deterministic condition of \eqref{eq:safeguard_optimization} can be verified through a stricter condition, i.e.
\begin{align}
\label{eq: upper condition}
    \mathbf{U}_f(x, u) < \max(\phi(x)-\eta, 0).
\end{align}


In Lemma \ref{lem: one step prediction}, we derive the probabilistic upper bound of safety index as
\begin{align}
    \mathbf{U}_f(x,u) := \phi(\mu_f(x, u)) + L_\phi \beta_f \sigma_f(x, u),
\end{align}
where $L_\phi$ is the Lipschitz constant of $\phi(\cdot)$ with respect to 1-norm. Lemma \ref{lem: one step prediction} shows that $\phi(f(x, u))$ is smaller than $\mathbf{U}_f(x,u)$ with probability at least ($1- \delta$). The proof of this probabilistic upper bound is given in \Cref{sec: one step prediction}.

\paragraph{Nonempty Set of Safe Control}
By introducing $\mathbf{U}_f(x,u)$, we have addressed the challenge (1.b). In the following two subsections, we will address challenge (1.a) by ensuring the existence of nonempty set of safe control for all possible states under control limit when solving \eqref{eq: upper condition}, i.e. 
\begin{align}
\label{eq: fundamental upper bound existence}
    \forall x \in \mathcal{X}, \exists u \in \mathcal{U}, \text{ s.t. }  \mathbf{U}_f(x, u) < \max(\phi(x)-\eta, 0).
\end{align}


\subsection{Infinite to Finite Conditions}
Notice that verifying condition \eqref{eq: fundamental upper bound existence} on the continuous state space is still intractable.
Therefore, we consider a discretization of the state space defined as follows.
\begin{definition}
[Discretization]
\label{definition:discretization}
A $\tau$-discretization $\mathcal{H}_\tau$ of a set $\mathcal{H}$ is defined as $\mathcal{H}_\tau := \{h_{1}, h_{2}, \ldots\}$ such that $\forall h \in \mathcal{H}, \exists h_i \in \mathcal{H}_\tau \text{ s.t. } ||h_i - h||_1 \leq \tau$.
\end{definition}
\begin{definition}[Data]
\label{Data}
A dataset on a state space $\tau$-discretization $\mathcal{X}_\tau$ is a collection of transition samples defined as $\mathcal{D}_\tau:=\{\big( x_{i}, u_{i} , 
f(x_{i}, u_{i})
\big)\}_{i=1}^{|\mathcal{X}_\tau|}$ where $x_i\in \mathcal{X}_\tau$.
\end{definition}
Given this discretization, if we ensure the existence of safe control for states in $\mathcal{X}_\tau$, together with the Lipschitz continuity and the bound on posterior variance of statistical models, then we can ensure the existence of safe control on the continuous state space $\mathcal{X}$.
\begin{proposition}[Equivalence in Feasibility Conditions]
\label{proposition:safe_control_discretization}
With the GP defined in Definition \ref{definition:gaussian_process}, the state-space $\tau_x$-discretization $\mathcal{X}_{\tau_x}$ defined in Definition \ref{definition:discretization} and the dataset $\mathcal{D}_{\tau_x}$ defined in Definition \ref{Data}, if the following condition holds: 
\begin{align}
\label{eq: stricter condition}
    \forall (x_{i}, u_i),
    \mathbf{U}_f(x_{i}, u_i) <&  \max\{\phi(x_{i})-\eta, 0\} - L_\phi L_f \tau_x - L_\phi \tau_x - 2L_\phi\beta_f\tilde{\sigma}_f~,
\end{align}
where 
\begin{align*}
    \tilde{\sigma}_f = n_x \sqrt{2L_k \tau_x +  2|\mathcal{X}_{\tau_x}| L_k \tau_x \|K^{-1} \| \max_{w, w' \in \mathcal{W}}k(w, w')},
\end{align*}
then it holds with probability $1-\delta$ that
\begin{equation*}\label{eq:safe_control_discretization}
\begin{split}
    \forall x \in \mathcal{X}, \exists u \in \mathcal{U}, \text{ s.t. } \mathbf{U}_f(x, u) < \max(\phi(x)-\eta, 0).\\
\end{split}
\end{equation*}
\end{proposition}
The proof of Proposition \ref{proposition:safe_control_discretization} is given in \Cref{sec: proof prop nonempty set of safe control}. Proposition \ref{proposition:safe_control_discretization} states that, in order to provide guarantee on the nonempty set of safe control in the whole continuous state space $\mathcal{X}$, it is sufficient to check a stricter condition (i.e., \eqref{eq: stricter condition}) of nonempty set of safe control on the discretized state set $\mathcal{X}_{\tau_x}$.
Note that the additional bounds on discretized states $\mathcal{X}_{\tau_x}$ (i.e. $L_\phi L_f \tau_x$, $L_\phi \tau_x$, $2L_\phi\beta_f\tilde{\sigma}_f$) of \eqref{eq: stricter condition} become zero as the discretization constant $\tau$ goes to zero.

\subsection{Dynamics Learning and Safety Index Design Theory}

\paragraph{Synthesize Safe Index} 
So far, we have shown that  \eqref{eq: stricter condition} implies
\eqref{eq: fundamental upper bound existence} in a probabilistic way.
Therefore, a theory that quantifies how to parameterize $\phi$ to make \eqref{eq: stricter condition} hold is needed.
To begin with, we first need to ensure there exists such a safety index to make \eqref{eq: stricter condition} hold. Hence, an assumption is made:
\begin{asm}[Safe Control]
\label{asm:safe_control}
The state space is bounded, and the infimum of the supremum of $\Delta \dot d$ can achieve positive, i.e., $\inf_x\sup_u\Delta \dot d(x,u) > 0$.
\end{asm}

Here $\Delta \dot d$ denotes the change of $\dot d$ at one time step. The necessity of \Cref{asm:safe_control} is summarized in \Cref{sec: necessity of assumption 3}. Essentially, \Cref{asm:safe_control} enables a degree two system to dissipate energy (i.e., $\ddot{\phi} < 0$) at all states. Subsequently, the safety index design rule is summarized as follows:

\begin{theorem}[Feasibility of Safety Index Design]
\label{theorem:main}
Denote $d(\cdot)$ and $\dot{d}(\cdot)$ as the mappings from $x$ to $d$ and $\dot{d}$ with Lipschitz constant $L_{d_x}$ and ${L}_{\dot{d}_x}$ with respect to 1-norm.
Under \Cref{asm:safe_control}, if we (1) select a state-space $\tau_x$-discretization $\mathcal{X}_{\tau_x}$ with step size such that 
\begin{equation}
\label{eq: the choice of tau}
    \resizebox{0.9\hsize}{!}{
    $\tau_x \leq \min \Bigg\{1, \bigg[\frac{\inf_{x}\sup_{u}\Delta\dot{d}(x,u)}{2(L_{d_x} + L_{\dot{d}_x})\big(1 + L_f + 2\beta_fn_x\sqrt{2L_k}\sqrt{1 +  |\mathcal{X}_{\tau_x}| \|K^{-1} \| \max_{w, w' \in \mathcal{W}}k(w, w')}\big)}\bigg]^2\Bigg\}$
}
\end{equation}
(2) construct the corresponding dataset $\{\big( x_{i}, u_{GP,i} , 
f(x_{i}, u_{GP,i})
\big)\}_{i=1}^{|\mathcal{X}_{\tau_x}|}$ on $\mathcal{X}_{\tau_x}$ by selecting $u_{GP,i}$ such that for any $x_i\in\mathcal{X}_{\tau_x}$
\begin{equation}
\label{eq: the choice dataset}
\underbrace{\dot{d}(f(x_i, u_{GP,i}))}_{\dot d_{GP,i}} - \underbrace{\dot{d}(x_i)}_{\dot d_{i}} > \frac{\inf_{x}\sup_{u}\Delta\dot{d}(x,u)}{2},
\end{equation}
(3) choose the safety index parameters such that

\begin{equation}
\label{eq: the choice of safety index}
\begin{cases}
      \sigma = 0,\\
      n = 1,\\
      k >  \max_{x_i \in \mathcal{X}_{\tau_x}} \big\{ \max\big\{1,\Upsilon_i
      \big\} \big\}
    \end{cases}
\end{equation}
where we denote $d_{GP,i} = d(f(x_i, u_{GP,i}))$, $d_{i} = d(x_i)$, and
\begin{align*}
    \Upsilon_i = \frac{\eta + d_{i} - d_{GP,i}}{\dot d_{GP,i} - \dot d_{i} - (L_{d_x} + L_{\dot{d}_x})\big(\tau_x - L_f \tau_x -2\beta_f n_x \tilde{\sigma}_f\big)} \\ \nonumber 
    \tilde{\sigma}_f = n_x \sqrt{2L_k \tau_x +  2|\mathcal{X}_{\tau_x}| L_k \tau_x \|K^{-1} \| \max_{w, w' \in \mathcal{W}}k(w, w')},
\end{align*}
then there always exists a  safe control for any discretized state
\begin{align}
\label{eq: the discretized nonempty set of safe control}
&\forall x_i \in \mathcal{X}_\tau, \; \exists u\in \mathcal{U},  \text{ s.t. } \\ \nonumber 
&\mathbf{U}_f(f(x_i, u)) < \max\{\phi(x_i)-\eta, 0\} - L_\phi L_f \tau_x - L_\phi \tau_x - 2L_\phi \beta_f \tilde{\sigma}_f.    
\end{align}
\end{theorem}

The proof for \Cref{theorem:main} is summarized in
\Cref{Proof of theorem:main}
. \Cref{theorem:main} states that, firstly, we select a proper discretization gap of state space such that it is small enough according to \eqref{eq: the choice of tau}. Secondly, we construct an offline dataset such that the selected control for each discretized state can increase $\dot d$ by a certain volume according to \eqref{eq: the choice dataset}. 
Lastly, by performing GP regression on the constructed dataset, 
the safety index designed according to \eqref{eq: the choice of safety index} ensures the existence of probabilistic safe control for all discretized states to satisfy \eqref{eq: the discretized nonempty set of safe control}. 
Note that \eqref{eq: the discretized nonempty set of safe control} is equivalent to \eqref{eq: stricter condition}, the following theorem is thus a direct consequence of  Proposition \ref{proposition:safe_control_discretization} and \Cref{theorem:main}.
\begin{theorem}
\label{theo: nonempty set of safe control}
Under the same assumptions of \Cref{theorem:main}, by selecting state discretization step size according to \eqref{eq: the choice of tau}, constructing Gaussian process dataset according to \eqref{eq: the choice dataset}, and defining safety index according to \eqref{eq: the choice of safety index}, then it holds with probability $1-\delta$ that
\begin{align}
\label{eq: real upper bound safety index}
    &\forall x \in \mathcal{X}, \; \exists u, \text{ s.t. } \\ \nonumber 
    &\phi(f(x, u)) \leq \mathbf{U}_f(x, u) < \max(\phi(x)-\eta, 0).
\end{align}
\end{theorem}


\begin{remark}
    It is worth noting that the system property $\inf_x\sup_u\Delta \dot d(x,u) > 0$ is crucial for establishing the nonempty set of safe control theorem as indicated in \eqref{eq: the choice of tau} and \eqref{eq: the choice dataset}. In practice, a lower bound of $\inf_x\sup_u\Delta \dot d(x,u)$ can be obtained via sampling the state space and control space, which is summarized in \Cref{sec: inf sup safety status}.
\end{remark}

\section{Uncertainty-Aware Implicit Safe Set Algorithm}

In the previous section, we established theoretical results for safety index design to ensure a nonempty set of safe control with learned dynamics models. However, due to the non-control-affine nature of the GP dynamics model and the limitations of conventional QP-based projection methods, we employ a multi-directional line search approach to solve the black-box optimization problem in \eqref{eq:safeguard_optimization}. In this section, we present a practical algorithm called Uncertainty-Aware Implicit Safe Set Algorithm (UAISSA) that builds upon the theoretical foundations discussed earlier and utilizes a sample-efficient black-box constrained optimization algorithm \citep{zhao2021issa}. The details of UAISSA can be found in \Cref{alg:uaissa} (see \Cref{sec: algo appdx}).
To have a better understanding on the \textit{Offline Stage} of UAISSA, we summarize the procedure for constructing a valid safety index and the associated GP dynamics model in \Cref{fig:uaissa_offline}. Firstly, 
we randomly select a step size $\tau$, and perform $\tau$-discretization of the state space.
For each discretized state $x_i$, we use sampling (grid sampling or random sampling) to find a control $u_{GP,i}$ satisfying \eqref{eq: the choice dataset}, which results in a dynamics learning dataset $\{\big( x_{i}, u_{GP,i} , 
f(x_{i}, u_{GP,i})
\big)\}_{i=1}^{|\mathcal{X}_{\tau_x}|}$. Next, we learn a GP dynamics model from the constructed dataset. Together with the Lipschitz constants and well-calibrated GP dynamics model, we can then evaluate \eqref{eq: the choice of tau}. If \eqref{eq: the choice of tau} does not hold, we will further shrink the discretization step size by half, and repeat the aforementioned procedures. If \eqref{eq: the choice of tau} holds, we will evaluate $\Upsilon_i$ for each $x_i$ from the dataset, and select the parameters for the safety index $\phi(x)$
according to \eqref{eq: the choice of safety index}.


With the guarantee of nonempty set of safe control provided by \Cref{theo: nonempty set of safe control}, and the fact that ISSA can always find a suboptimal solution of \eqref{eq:safeguard_optimization} with finite iterations if the set of safe control is non-empty [Proposition 2,~\citep{zhao2021issa}], the following theorem is thus a direct consequence of Theorem 1 from~\citep{zhao2021issa}.

\begin{theorem}[Forward Invariance]
\label{thoem:forward_invariance}
If the control system satisfies \Cref{asm:safe_control} and with the GP model and the safety index as specified in \Cref{theorem:main}, then if $\phi(x_t) \leq 0$, \Cref{alg:uaissa} guarantees $\phi(x_{t+1}) \leq 0$ with probability $1 - \delta$.
\end{theorem}


\section{Experiment}

We evaluate UAISSA in two experiments: (i) Robot arm, where we apply 
 \Cref{theo: nonempty set of safe control} to ensure nonempty set of safe control for an unknown robotics manipulator system
; (ii) Safety Gym, where we apply UAISSA to safeguard unknown complex systems. 

\subsection{Robot Arm}

\begin{wrapfigure}{r}{0.2\textwidth}
\vspace{-10pt}
    \centering
    \includegraphics[width=0.18\textwidth]{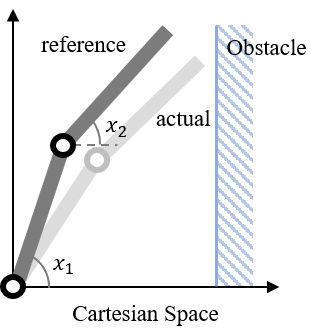}
    \caption{2DOFs robot manipulator.}
    \label{fig:2DOFS robot arm}
\vspace{-10pt}
\end{wrapfigure}
We verify the correctness of our approach on a planar robotics manipulator with 2 degrees of freedom (2DOFs)~\citep{zhao2022provably}. The robot has a four dimensional state space:  $x =[\theta_1, \theta_2, \dot\theta_1, \dot\theta_2]$, where $\theta_i$ is the $i$-th joint angle in the world frame. We consider limited state space, i.e., $\theta_1 \in [0, \pi]$, $\theta_2 \in [0, 2\pi]$, $\forall i = 1,2, \dot\theta_i \in [-0.1, 0.1]$. 
The system inputs are accelerations of the two joints, i.e. $[\ddot \theta_1, \ddot \theta_2]$. 
The system is simulated with $dt = 1$ms. The system is shown in \Cref{fig:2DOFS robot arm}, where the robot is randomly exploring the environment and we need to safeguard the robot from colliding with the wall. 
The link length of the robot is 1 meter. The wall is
1 meter
away from the robot base.

\subsubsection{Safety Index Design Running Example}
To apply \Cref{theo: nonempty set of safe control} to obtain the safety index parameters, we consider $L_f, L_{\Delta \dot d}, L_{d_x}, L_{\dot d_x}$ to be known.
Firstly, we need to find the proper state space
disretization step size $\tau_x$. We start with $\tau_x = 0.5$, and construct a learning  dataset where a safe control is sampled for each discretized state, such that \eqref{eq: the choice dataset} holds. dynamics data sample include input entry and output entry, where the input entry is a stack of state and sampled control ($[\theta_1, \theta_2, \dot\theta_1, \dot\theta_2, \ddot \theta_1, \ddot \theta_2]$), and output entry is the state at next time step. An example for data sample is: $\{[0.1, 0.5    ,-0.1, -0.1, 0.82, 0.36], [0.1,  0.49, -0.09, -0.09]\}$.

Then, we perform Gaussian Process to learn a well-calibrated dynamics model, where a uniform error bound theory (Lemma \ref{lem:uniform_error_bound_with_guarantee} in \Cref{sec: munich}) with $\delta = 1\%$ (i.e., 99\% confidence interval) is applied to select $\beta_f$. With the learnt GP model, we check if \eqref{eq: the choice of tau} holds. If not, we further decrease $\tau_x$ by multiplying $\tau_x$ with 0.99 and repeat the process. 

Finally, we find a discretization step of $\tau_x = 0.174$, resulting a dataset with $2516$ samples. By setting $\eta = 0.05$, the safety index parameterization is obtained as: $\sigma = 0, n = 1, k = 2.54$ according to \eqref{eq: the choice of safety index}. Intuitively, $k$ reflects UAISSA reaction sensitivity to unsafe situations, e.g. larger $k$ indicates safe control is more likely to be generated when the robot moves toward the obstacle.


\subsubsection{Robot Arm Results}
This section numerically verifies that the synthesized safety index facilitates probabilistic forward invariance, by showing that 1) the upper bound $\mathbf{U}_f$ of the safety index is a true upper bound; 2) there is always a feasible control that satisfies the constraint in \eqref{eq:safeguard_optimization}.

\begin{figure}
    \centering
    \includegraphics[width=0.5\columnwidth]{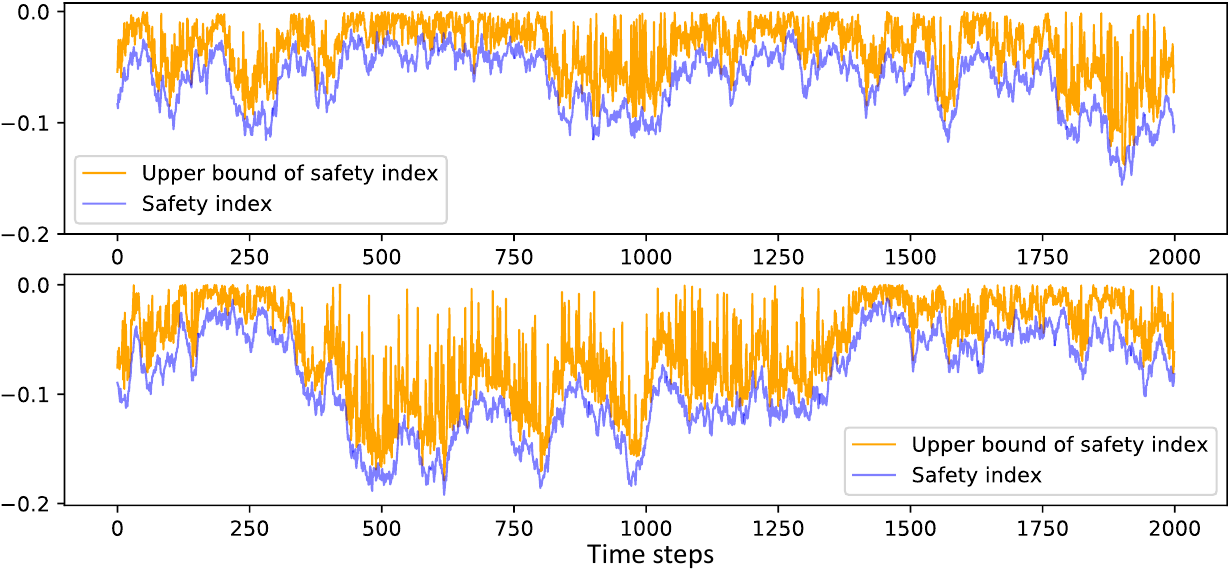}
    \vspace{-10pt}
    \caption{Evolutions of safety index and its upper bound with UAISSA over two runs.}
    \label{fig:robotarm_upperbound}
    \vspace{-10pt}
\end{figure}

We simulate the system for $2000$ time steps with the safety index parameters designed above, and the evolution of $\mathbf{U}_f(x, u)$ (orange curves) and $\phi(f(x, u))$ (blue curves) is summarized in \Cref{fig:robotarm_upperbound}. Overall, by ensuring $\mathbf{U}_f(x, u) < \max(\phi(x)-\eta, 0)$, UAISSA ensures $\phi(f(x, u)) < \max(\phi(x)-\eta, 0)$ along the simulations. As shown in \Cref{fig:robotarm_statespace}, with the safety index synthesized using \eqref{eq: the choice of safety index}, the nonempty set of safe control for all possible states are guaranteed. 
\begin{wrapfigure}{r}{0.4\textwidth}
\vspace{-10pt}
    \centering
    \includegraphics[width=0.4\textwidth]{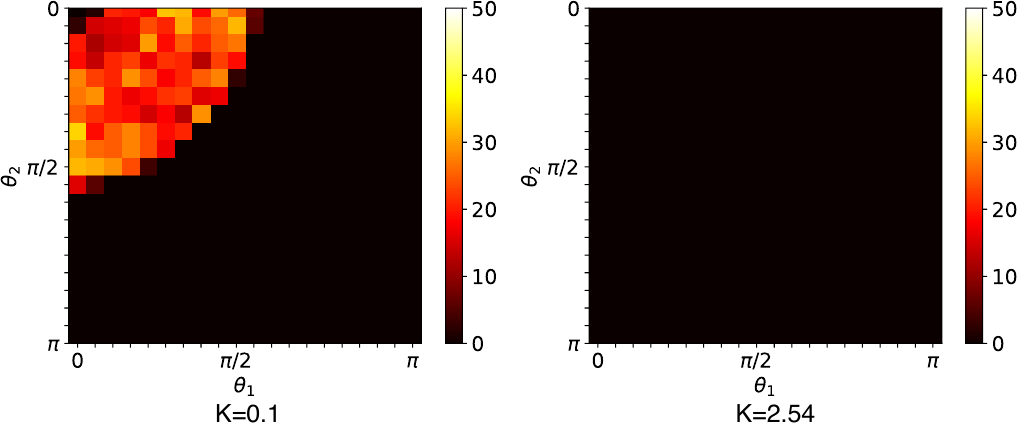}
    \caption{Distribution of states with infeasible safe controls when we optimize the safety index. Each grid in the graph corresponds to a position of joint angles $(\theta_1, \theta_2)$. We sample 100 states at each position (with different velocities of joint angles). The color denotes how many states at this position has an empty set of safe control. The left shows that a randomly selected safety index ($k=0.1$) results in empty set of safe control for many system states. The right shows that our synthesized safety index ($k=2.54$) ensures that we can always find a feasible safe control.}
    \label{fig:robotarm_statespace}
\vspace{-30pt}
\end{wrapfigure}
Furthermore, We conduct an ablation study on different discretization gap $\tau$. The results are summarized in \Cref{fig:different_tau_robotarm} (\Cref{sec: alation study}), where the gap between the upper bound of the safety index and the safety index decreases with smaller discretization gaps. This result validates our theoretical results as smaller discretization gaps result in smaller error bounds of the safety index. In practice, we believe a smaller discretization gap is beneficial to the performance of robot controllers since more accurate estimates of $\mathbf{U}_f(x,u)$ alleviate the performance drop caused by conservative safeguards. However, note that smaller discretization gaps also result in large datasets which may be computationally expensive for GP. It is a trade-off between lower computational cost and better performance.


\subsection{Safety Gym}
\begin{figure}[t]
    \centering
    \includegraphics[width=120mm]{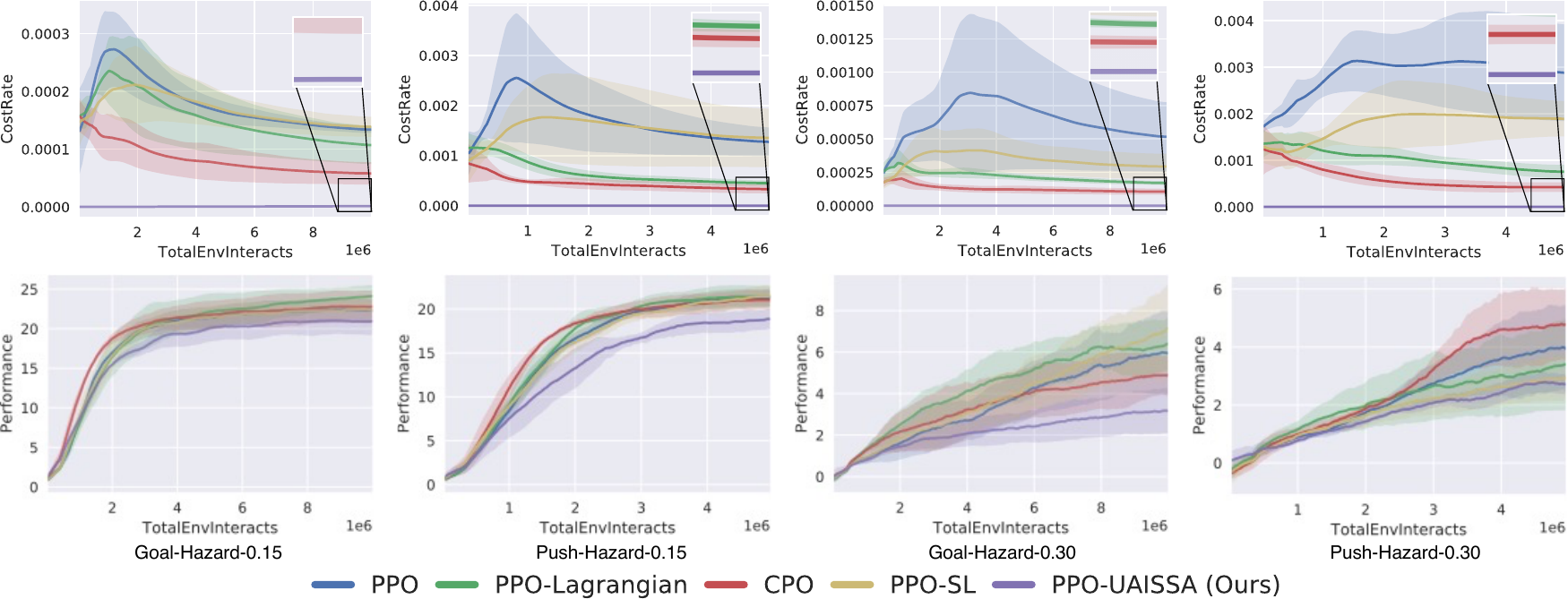}
     \caption{Cost rates and rewards of UAISSA and baselines on Safety Gym benchmarks with different tasks and sizes of the hazard over five random seeds.}
     \label{fig:goal_and_push}
     \vspace{-15pt}
\end{figure}

\paragraph{Scale to High-dimensional Environments}
One drawback of GP is that it scales very badly with the number of observations.
To scale UAISSA to hign-dimensional environments, we propose to use deep Gaussian Process~\citep{gal2016dropout} as an approximation of GP for dynamics learning. Note that the scalability comes with the price of losing theoretical safety guarantee because the uniform error bound of GP no longer holds when we use deep GP. Nevertheless, this section shows that UAISSA empirically achieves near zero-violation safety performance with deep GP.
\paragraph{Environment Setting}To test how UAISSA safeguards RL policies in high-dimensional complex environments, we conduct experiments on the widely adopted benchmark of safe RL, Safety Gym~\citep{ray2019benchmarking}. We evaluate UAISSA on two different control tasks (i.e., Goal-Hazard and Push-Hazard), where the environment settings are introduced in \Cref{appendix: experiment details}. 
\vspace{-5pt}
\paragraph{Baseline Selection}We choose PPO~\citep{schulman2017ppo} as the nomial RL algorithm and add UAISSA as a safeguard (namely PPO-UAISSA) on top of the nominal RL policy.
We compare UAISSA with (i) PPO~\citep{schulman2017ppo} (stardard RL algorithm); (ii) PPO-Lagrangian \citep{chow2017risk} and CPO~\citep{achiam2017cpo} (safe RL algorithms); (iii) PPO-SL~\citep{dalal2018safe} (RL with a different safeguard).
\vspace{-5pt}
\paragraph{Policy Settings}
Detailed parameter settings are summarized in \Cref{tab:policy_setting} (\Cref{appendix: experiment details}). All the policies in our experiments use the default hyper-parameter settings hand-tuned by Safety Gym~\citep{ray2019benchmarking} except that we set the $\text{cost limit} = 0$ for PPO-Lagrangian and CPO since the goal is to achieve zero-violation performance.
\vspace{-5pt}
\paragraph{Evaluation Results}
The evaluation results are shown in \Cref{fig:goal_and_push}, where PPO-UAISSA achieves near zero violation while gaining comparable rewards on both tasks. Note that the violations made by PPO-UAISSA are so few (nearly 1\% of violations made by standard PPO), making it hard to observe in \Cref{fig:goal_and_push}. Such results align with our probabilistic safety guarantee given in \Cref{thoem:forward_invariance}.
As for safe RL methods, both CPO and PPO-Lagrangian fail to achieve zero violation even with a cost limit of zero.
PPO-SL proposed in ~\citep{dalal2018safe}  also uses learned dynamics with an offline dataset, but PPO-SL failed to reduce safety violation due to (i) the assumption of linear cost functions is unrealistic in complex environments like MuJoCo~\citep{todorov2012mujoco}; 
(ii) the lack of quantification of the error bound from neural networks.
More experiments details,  comparison metrics  and experimental results are summarized in \Cref{sec:metrics}.

\vspace{-10pt}
\section{Conclusion}
This paper presented a safe control framework with a learned dynamics model using Gaussian process. The proposed theory guarantees (i) the nonempty set of safe control for all states under control limits, and (ii) probabilistic forward invariance to the safe set.
Simulation results on a robot arm and Safety Gym show near zero violation safety performance. One limitation of our work is that offline synthesis requires grid-based discretization of state space, which is computationally expensive for high-dimensional system. In the future work, we are going to investigate how to speed up offline synthesis, such as parallelization computation.


\bibliography{main}  

\clearpage
\appendix

\section{Related Work}
\label{sec: related work}
The desire to make learning agents persistently satisfy a state constraint drives many safe learning works to study state-wise safety. Existing works can be divided into three categories.
\begin{enumerate}
    \item The first category of works assumes the knowledge of white-box or black-box dynamics models, and achieve safe exploration by safeguarding the control actions generated by the RL agent~\citep{ferlez2020shieldnn,fisac2018general,cheng2019end,zhao2021issa}.
    Specifically, ShieldNN~\citep{ferlez2020shieldnn} designs a safety filter embedded into a neural network to achieve safety guarantees. However, ShieldNN is specially designed for an environment with the kinematic bicycle model (KBM) \citep{kong2015kinematic}, which is hard to generalize to other problem scopes. 
    Implicit safe set algorithm (ISSA)~\citep{zhao2021issa} filters out unsafe control actions by projecting them to a set of safe control, which is computed by checking the safety level of the resulting next state using one-step simulation. However, the forward simulation requires a black-box dynamics model that is hard to obtain in practice. 
    \item The second category of works achieves safety using online learned models. A safe learning framework based on Hamilton-Jocabi reachability methods is proposed to safeguard policy learning~\citep{li2020generating} with online learned GP dynamics models. However, since the learner has limited control over the data distribution during online learning,
    the policy may become very conservative in exchange for safety guarantees.
    Another model-based safe RL framework \citep{berkenkamp2017safembrl} is proposed to safely explore the environment with online learned dynamics models.
    Nevertheless, the framework requires a proper Lyapunov function that can be dissipated at all states and an initial safe policy, which are both challenging to get in complex environments. 
    \item The last category of methods provides safety guarantees based on learned models from offline datasets. ~\citep{dean2019safely} studies how to construct a dataset, and how to learn the dynamics such that constrained linear quadratic regulator system can be stabilized. However, only linear systems are studied, which limits the applicability to complex nonlinear systems. Some other methods propose to safeguard RL agents based on learned dynamics using neural networks~\citep{dalal2018safe} or Gaussian process~\citep{cheng2019end}. However, due to the lack of quantification of uniform error bounds of statistical models, such methods~\citep{dalal2018safe,cheng2019end} can not provide formal safety guarantees since the prediction error of the dynamic model may be arbitrarily big. 
    Our work can be classified into the third category. To the best of our knowledge, our paper is the first to provide a formal
    safety guarantee for nonlinear systems based on offline datasets.
\end{enumerate}


\section{Algorithm}
\label{sec: algo appdx}
\begin{algorithm}
\caption{Uncertainty-Aware Implicit Safe Set Algorithm (UAISSA)}
\label{alg:uaissa}
\begin{algorithmic}[1]
\Procedure{UAISSA}{$\pi$} 
\State \textbf{Offline Stage:}
\State Select state discretization step size $\tau$ according to \eqref{eq: the choice of tau}
\State Construct dynamics learning dataset on discretized states satisfying \eqref{eq: the choice dataset}
\State Perform Gaussian process on the dataset and choose safety index according to \eqref{eq: the choice of safety index}
\State \textbf{Online Stage:} 
\State \textbf{for} $t = 0, 1, 2, \cdots$ \textbf{do}
\State \qquad Obtain reference control $u_t^r \leftarrow \pi(x_t)$
\State \qquad Solve \eqref{eq:safeguard_optimization} to obtain safe control $u_t$ via Implicit Safe Set Algorithm (ISSA) [Algorithm 2,~\citep{zhao2021issa}], s.t. safety status of $u_t$ is \textit{SAFE} (i.e., $\mathbf{U}_f(x_t, u_t) < \max(\phi(x_t)-\eta, 0)$)
\State \qquad Apply $u_t$ to the control system
\EndProcedure
\end{algorithmic}
\end{algorithm}

\newpage

\section{Theoretical Details}

\subsection{Error Bound of One-step Safety Index Prediction}
\label{sec: one step prediction}
\begin{lemma}
\label{lem: one step prediction}
With the well-calibrated model, 
the one-step prediction error of the safety index is bounded with probability $1-\delta$, i.e.,

\begin{align}
\label{eq:error_bound_safety_index}
    &\forall x \in \mathcal{X}, P(|\phi(f(x, u)) - \phi(\mu_f(x, u))| \leq |L_\phi \beta_f \sigma_f(x ,u)|) \geq 1-\delta.
\end{align}
\begin{proof}
With the well-calibrated model, according to Lemma \ref{lem:Well-calibrated}, we have that with probability at least $1 - \delta$,
\begin{equation}\label{eq:error_bound_safety_index_proof}
\begin{split}
    |\phi(f(x, u)) - \phi(\mu_f(x, u))|
    \leq & |L_\phi ||f(x, u) - \mu_f(x, u)||_1 | \\
    \leq & |L_\phi \beta_f \sigma_f(x, u)| \\
\end{split}
\end{equation}
\end{proof}
\label{lem:error_bound_safety_index}
\end{lemma}

\subsection{Necessity of Assumption \ref{asm:safe_control}}
\label{sec: necessity of assumption 3}
\begin{remark}
We show the necessity of \Cref{asm:safe_control} by contradiction. Define a subset of state space where $\mathcal{X}_d = \{x_t\in \mathcal{X} | \phi(x_t) -\eta>0,\dot{d_t} \leq 0\}$. If the system does not satisfy \Cref{asm:safe_control}, which means there might exists $ x_t \in \mathcal{X}_d$, such that $\forall u_t \in \mathcal{U},  \Delta\dot d(x_t,u) \leq 0$, hence
\begin{align}
\label{eq: delta_phi_contradiction}
\forall u \in \mathcal{U}, &\phi(x_{t+1})-\phi(x_t) \\ \nonumber 
& = (d_{min} - d_{t+1}^n - k\dot{d}_{t+1} ) - (d_{min} - d_t^n - k\dot{d_t} )\\ \nonumber
& = (d_{t}^n - d_{t+1}^n) - (k\dot{d_{t}} - k\dot d_{t+1})\\\nonumber
& \geq 0 - k \cdot \Delta\dot d(x_t,u) \\\nonumber
& \geq 0,
\end{align}
which indicates the set of safe control for $x_t$ is empty:
\begin{align}
\label{eq: empty_set_of_safe_control}
    \forall u \in \mathcal{U}, \phi(x_{t+1}) \geq \max\{\phi(x_t)-\eta, 0\}.
\end{align}
Therefore, if the system does not satisfy \Cref{asm:safe_control}, we cannot rule out the possibility of empty set of safe control for certain states, regardless of the safety index design.
Therefore, \Cref{asm:safe_control} is necessary for providing guarantee on the nonempty set of safe control for all states.
\end{remark}

\subsection{Upper Bound of Posterior Variance Function}
\begin{lemma}
\label{lem:bound_posterior_variance}
Consider a zero mean Gaussian process $g: \mathcal{Z} \rightarrow \mathbb{R}$ defined through the continuous covariance kernel k(·, ·) with Lipschitz constant $L_k$ on the compact set $\mathcal{Z}$. Let $\mathcal{Z}_\tau := \{{z}_1, {z}_2, \cdots,  {z}_N\}$ be a $\tau$-discretization set of $\mathcal{Z}$. Consider observations $({z}_i,g(z_i))$ of $g$ over $\mathcal{Z}_\tau$. The posterior variance function $\sigma^2_g$ is then bounded by a constant $\tilde{\sigma}^2_g$:
\begin{align}
\label{eq:bound_posterior_variance}
    &\forall z_* \in \mathcal{Z}, \\ \nonumber 
    &\sigma^2_g(z_*) \leq \tilde{\sigma}^2_g := 2L_k \tau +  2N L_k \tau \|K^{-1} \| \max_{z, z' \in \mathcal{Z}}k(z, z').
\end{align}

\begin{proof}
Since $\sigma_g$ is a scalar, the difference of absolute values of the posterior variance at $z_*$ and $z_i \in \mathcal{Z}_\tau$ can be expressed as:
\begin{align}
    \label{eq:variance_bound_1}
    |\sigma_g^2(z_*) - \sigma_g^2(z_i) | &=  |\sigma_g(z_*) + \sigma_g(z_i) | \cdot |\sigma_g(z_*) - \sigma_g(z_i) |.
\end{align}
Note that the standard deviation is positive semidefinite (i.e., $\sigma \geq 0$), we have
\begin{align}
    \label{eq:variance_bound_3}
    |\sigma_g(z_*) - \sigma_g(z_i) |^2 \leq |\sigma_g^2(z_*) - \sigma_g^2(z_i) |.
\end{align}
With \eqref{eq:variance_bound_3}, we can bound the difference of standard deviation by taking the square root of the the difference of variance. With the Cauchy-Schwarz inequality and 1-norm Lipschitz constant $L_k$ of kernel function $k(\cdot, \cdot)$, the 1-norm difference of the variance can be bounded by
\begin{align}
    \label{eq:variance_bound_4}
    & \|\sigma_g^2(z_*) - \sigma_g^2(z_i) \|_1  \\ \nonumber 
    = & \| k(z_*, z_*) - k_*^T(z_*)K^{-1}k_*(z_*) \\ \nonumber
    & - k(z_i, z_i) + k_*^T(z_i)K^{-1}k_*(z_i) \|_1 \\ \nonumber  
    \leq & \big\| 2L_k\|z_*-z_i\|_1 \quad \leftarrow \texttt{Lipchitz } L_k \\ \nonumber 
    & + \|k_*(z_*)-k_*(z_i)\|_2 \|K^{-1}\|_2 \|k_*(z_*)+k_*(z_i)\|_2 \big\|_1. \\ \nonumber
    & \quad\quad\quad\quad\quad \leftarrow \texttt{Cauchy-Schwarz inequality}
\end{align}
where $\| K^{-1}\|$ denotes the Frobenius norm of the matrix $ K^{-1}$. Note that the first term of RHS of \eqref{eq:variance_bound_4} can be bounded as
\begin{align}
    \label{eq:variance_bound_5}
    & 2L_k\|z_*-z_i\|_1 \\ \nonumber
    \leq & 2L_k \tau.
\end{align}
As for the second term of RHS of \eqref{eq:variance_bound_4}, on one hand we have
\begin{align}
    \label{eq:variance_bound_6}
    & \|k_*(z_*)-k_*(z_i)\|_2 \\ \nonumber
    \leq & \sqrt{N}L_k\| z_* - z_i\|_1\\ \nonumber
    \leq & \sqrt{N} L_k \tau,
\end{align}
where $N$ is the number of data points for GP.
On the other hand, we have
\begin{align}
    \label{eq:variance_bound_7}
    & \|k_*(z_*)+k_*(z_i)\|_2 \\ \nonumber
    \leq & 2\sqrt{N}\max_{
    z, z' \in \mathcal{Z}}k(z, z').
\end{align}
Substituting \eqref{eq:variance_bound_5}, \eqref{eq:variance_bound_6} and \eqref{eq:variance_bound_7} into \eqref{eq:variance_bound_4}, we have
\begin{align}
    \label{eq:variance_bound_8}
    & \|\sigma_g^2(z_*) - \sigma_g^2(z_i) \|_1  \\ \nonumber 
    \leq & 2L_k \tau +  \sqrt{N} L_k \tau \times \|K^{-1} \| \times 2\sqrt{N}\max_{
    z, z' \in \mathcal{Z}}k(z, z') \\ \nonumber  
    \leq & 2L_k \tau +  2N L_k \tau \|K^{-1} \| \max_{z, z' \in \mathcal{Z}}k(z, z').
\end{align}
With \eqref{eq:variance_bound_3} and \eqref{eq:variance_bound_8}, we obtain the bound of standard deviation as 
\begin{align}
    \label{eq:variance_bound_9}
    & \|\sigma_g(z_*) - \sigma_g(z_i) \|_1 \\ \nonumber
    \leq & \sqrt{\|\sigma_g^2(z_*) - \sigma_g^2(z_i) \|_1} \\ \nonumber
    \leq & \sqrt{2L_k \tau +  2N L_k \tau \|K^{-1} \| \max_{z, z' \in \mathcal{Z}}k(z, z')}.
\end{align}
With \eqref{eq:variance_bound_9}, we further bound $\sigma_g(z_*)$ by the fact that $\sigma_g(z_i) = 0$ (because $(z_i, g(z_i))$ is a data point for GP) as follows.
\begin{align}
    \label{eq:variance_bound_10}
    & \sigma_g(z_*) \\ \nonumber
    = & \|\sigma_g(x,u_0) - \sigma_g(z_i) \|_1 \\ \nonumber
    \leq & \sqrt{2L_k \tau +  2N L_k \tau \|K^{-1} \| \max_{z, z' \in \mathcal{Z}}k(z, z')}.
\end{align}
For simplicity, we denote $\tilde{\sigma}_g = \sqrt{2L_k \tau +  2N L_k \tau \|K^{-1} \| \max_{z, z' \in \mathcal{Z}}k(z, z')}$. Note that $\tilde{\sigma}_g$ is a value determined by the GP model: (i) kernel function $k(\cdot, \cdot)$ of GP; (ii) data points for GP; (iii) GP discretization gap $\tau$.
\end{proof}
\end{lemma}

\subsection{Upper Bound of Posterior Mean Function}
\begin{lemma}
\label{lem:bound_posterior_mean}
Consider a zero mean Gaussian process $g: \mathcal{Z} \rightarrow \mathbb{R}$ defined through the continuous covariance kernel k(·, ·) with Lipschitz constant $L_k$ on the compact set $\mathcal{Z}$. Let $\mathcal{Z}_\tau := \{{z}_1, {z}_2, \cdots,  {z}_N\}$ be a $\tau$-discretization set of $\mathcal{Z}$. Consider observations $({z}_i,g(z_i))$ of $g$ over $\mathcal{Z}_\tau$,
the posterior mean function $\mu_g$ is bounded by:

\begin{align}
\label{eq:bound_posterior_mean}
    \mu_g \leq &\max\{g(z_1), g(z_2), \cdots, g(z_N)\} + \sqrt{N}L_k\tau \|K^{-1}y_N\|_2.
\end{align}

\begin{align}
    \mu_g \leq &\max\{g(z_1), g(z_2), \cdots, g(z_N)\} + \sqrt{N}L_k\tau \|K^{-1}y_N\|_2.
\end{align}
\end{lemma}
\begin{proof}

According to \eqref{eq:gp_mu_sigma}, we have the posterior mean function that satisfies:
\begin{align}
    \mu_g(z_*) &= k^T_*(z_*)K^{-1}y_N,
\end{align}
where $K_{i,j} = k(z_i, z_j)$ and $k_*(z_*) = [k(z_1, z_*), k(z_2, z_*), \cdots, k(z_N, z_*)]^T$.  With the Cauchy-Schwarz inequality and 1-norm Lipschitz constant $L_k$ of kernel function $k(\cdot, \cdot)$, the difference of absolute values of mean function can be bounded by:
\begin{align}
    \|\mu_g(z_*) - \mu_g(z_i)\|_1  &= \|(k^T_*(z_*) - k^T_*(z_i))K^{-1}y_N\|_1 \\ \nonumber 
    &\leq \|k^T_*(z_*) - k^T_*(z_i)\|_2\|K^{-1}y_N\|_2 \\ \nonumber 
    &\leq \sqrt{N}L_k\|z_*-z_i \|_1 \|K^{-1}y_N\|_2\\ \nonumber 
    &\leq \sqrt{N}L_k\tau \|K^{-1}y_N\|_2
\end{align}

Therefore, the posterior mean function $\mu_g$ is bounded by:
\begin{align}
    \mu_g \leq &\max\{g(z_1), g(z_2), \cdots, g(z_N)\} \\ \nonumber 
    &+ \sqrt{N}L_k\tau \|K^{-1}y_N\|_2.
\end{align}
\end{proof}

\subsection{Upper Bound of Posterior Variance of Multi-Dimensional Function}
\begin{lemma}
\label{lem:bound_posterior_variance_multidimensional}
Consider a zero mean Gaussian process $g: \mathcal{Z} \rightarrow \mathbb{R}^D$ defined through the continuous covariance kernel k(·, ·) with Lipschitz constant $L_k$ on the compact set $\mathcal{Z}$. Let $\mathcal{Z}_\tau := \{{z}_1, {z}_2, \cdots,  {z}_N\}$ be a $\tau$-discretization set of $\mathcal{Z}$. 
Consider observations $(z_i,g(z_i))$ of $g$ over $\mathcal{Z}_\tau$, the posterior standard deviation function $\sigma_g$ is bounded by $\tilde{\sigma}_g$:
\begin{align}
\label{eq:bound_posterior_variance_multidimensional}
    &\forall z_* \in \mathcal{Z}, \\ \nonumber 
    &\sigma_g(z_*) \leq \tilde{\sigma}_g \\ \nonumber 
    &\quad \quad \quad = D \big(2L_k \tau +  2N L_k \tau \|K^{-1} \| \max_{z, z' \in \mathcal{Z}}k(z, z')\big)^{\frac{1}{2}}.\\
\end{align}


\begin{remark}
\label{multidimensional_GP}
In the case of multiple output dimensions ($D > 1$), we consider GP regression on a function with one-dimensional output $g'(z,i): \mathcal{Z}\times\mathcal{I} \rightarrow \mathbb{R}$, with the output dimension indexed by $i \in \mathcal{I} = \{1,2,\cdots,D\}$. This allows us to use the standard GP methods to provide uniform error bounds with multi-dimensional outputs~\citep{berkenkamp2017safembrl}. Then, we define the posterior distribution of Gaussian process as $\mu_g(z) = [\mu_{g'}(z,1), \cdots, \mu_{g'}(z,D)]^T$ and $\sigma_g(z) = \sum_{1 \leq i \leq D}\sigma_{g'}(z, i)$.
\end{remark}
\begin{proof}
According to the remark, to bound $\sigma_g(z)$, we need to bound $\sigma_{g'}(z, j)$ of each dimension.
As more data samples for GP could only decrease the posterior variance~\citep{williams2006gaussian}, we can bound $\sigma_{g'}(z, j)$ with \textit{less data} to obtain an upper bound. 

As for bounding $\sigma_{g'}(z, j)$ with \textit{less data}, it means that we compute the bound for $\sigma_{g'}(z, j)$ solely with the data samples for $j$-th output dimension, i.e. $\mathcal{Z} \times j \rightarrow \mathbb{R}$, where $j \in \mathcal{I}$.


The proof of the bound for $\sigma_{g'}(z, j)$ is almost identical to the proof in Lemma \ref{lem:bound_posterior_variance}, except that we 
replace (i) $z_*$ with $(z_*, j)$, and (ii) $z_i$ with $(z_i, j)$. Then, we replace $\|z_* - z_i\|_1 \leq \tau$ in \eqref{eq:variance_bound_5} and \eqref{eq:variance_bound_6} with
\begin{align}
    \label{eq:from_z_to_z_j_multidimensional}
    \| (z_*, j) - (z_i, j)\|_1 = \| z_* - z_i\|_1 + \|j - j\|_1 \leq \tau.
\end{align}
For most of the common kernels, e.g. RBF kernel, $k_{z,z'\in\mathcal{Z}}((z,j),(z,j')) = k_{z,z'\in\mathcal{Z}}((z),(z))$, hence, $\sigma_{g'}(z, j)$ can be bounded as:
\begin{align}
    \label{eq:bound_of_z_j}
    \sigma_{g'}(z_*, j) \leq \big(2L_k \tau +  2N L_k \tau \|K^{-1} \| \max_{z, z' \in \mathcal{Z}}k(z, z')\big)^{\frac{1}{2}}.
\end{align}
where $K_{i,j} = k(z_i, z_j)$. We can further bound $\sigma_g$ based on the definition where $\sigma_g(z_*) = \sum_{j \leq D}\sigma_{g'}(z_*, j)$:

\begin{align}
\label{eq:bound_posterior_variance_multidimensional_inproorf}
    &\forall z_* \in \mathcal{Z}, \\ \nonumber 
    &\sigma_g(z_*) \leq \tilde{\sigma}_g \\ \nonumber 
    &\quad \quad \quad = D \big(2L_k \tau +  2N L_k \tau \|K^{-1} \| \max_{z, z' \in \mathcal{Z}}k(z, z')\big)^{\frac{1}{2}}.\\
\end{align}
which directly yields the result.
\end{proof}
\end{lemma}

\subsection{Proof of Proposition \ref{proposition:safe_control_discretization}}
\label{sec: proof prop nonempty set of safe control}
\begin{proof}
Denote $A = L_\phi L_f \tau_x$, $B = L_\phi \tau_x$, and $C =2L_\phi\beta_f\tilde{\sigma}_f$, where \\
$\tilde{\sigma}_f = n_x \sqrt{2L_k \tau_x +  2|\mathcal{X}_{\tau_x}| L_k \tau_x \|K^{-1} \| \max_{w, w' \in \mathcal{W}}k(w, w')}$. For any $x\in\mathcal{X}$, we denote $x_{\tau_0} = \argmin_{x_i \in \mathcal{X}_{\tau_x}}\|x_i - x\|_1$ as the closest discretized state. We further denote $(x_{\tau_0}, u_0)$ as a data sample from $\mathcal{D}_\tau$. Note that $\forall (x_i, u_i),
\mathbf{U}_f(x_i, u_i) < \max\{\phi(x_i)-\eta, 0\} - A - B - C$.


Next, We will prove Proposition \ref{proposition:safe_control_discretization} by showing  $\forall x \in \mathcal{X}, \mathbf{U}(x, u_0) < \max \{\phi(x)-\eta, 0\}$. We first expand the terms of $\mathbf{U}_f(x, u_0)$:
\begin{equation}\label{eq:U_f(x,u)}
\begin{split}
    \mathbf{U}(x ,u_0) 
    & = \phi(\mu_f(x,u_0)) + L_\phi \beta_f \sigma_f(x, u_0)\\
    & = \phi(\mu_f(x,u_0)) - \phi(f(x,u_0)) \\
    & \quad + \phi(f(x, u_0)) + L_\phi \beta_f \sigma_f(x, u_0)\\
\end{split}
\end{equation}

Based on Lipschitz constants and error bound in well-calibrated model, the following inequality derived from \eqref{eq:U_f(x,u)} holds with probability at least $1-\delta$:
\begin{equation}\label{eq:safe_control_discretization_proof_1}
\begin{split}
    \mathbf{U}(x ,u_0) 
    & \leq L_\phi \| \mu_f(x, u_0) - f(x, u_0) \|_1 \\
    & \quad + \phi(f(x, u_0)) + L_\phi \beta_f \sigma_f(x, u_0) \\ 
    &\quad\quad\quad\quad\quad\quad\quad \leftarrow \texttt{Lipchitz } L_\phi\\
    & \leq L_\phi \beta_f \sigma_f(x, u_0)  + \phi(f(x, u_0)) \\
    & \quad + L_\phi \beta_f \sigma_f(x, u_0) \\
    &\quad\quad\quad\quad\quad\quad\quad \leftarrow \texttt{GP error bound}\\ 
    & = \phi(f(x, u_0)) + 2L_\phi \beta_f \sigma_f(x, u_0) \\
    & = \phi(f(x, u_0)) - \phi(f(x_{\tau_0}, u_0)) + \phi(f(x_{\tau_0}, u_0)) \\ 
    & \quad + 2L_\phi \beta_f \sigma_f(x, u_0) \\
    & \leq L_\phi\|f(x, u_0) - f(x_{\tau_0}, u_0)\|_1 + \mathbf{U}_f(x_{\tau_0}, u_0) \\
    &\quad + 2L_\phi \beta_f \sigma_f(x, u_0) \quad \leftarrow \texttt{Lipchitz } L_f\\
    & \leq L_\phi L_f \|(x, u_0) - (x_{\tau_0}, u_0) \|_1 \\
    & \quad + \max\{\phi(x_{\tau_0})-\eta, 0\} - A - B - C \\
    & \quad + 2L_\phi \beta_f \sigma_f(x, u_0)
\end{split}
\end{equation}

By definition of the discretization, we have on each grid cell that
\begin{align}
    \label{eq:norm_tau}
    \|(x,u_0) - (x_{\tau_0}, u_0) \|_1 &= \|x - x_{\tau_0} \|_1 + \|u_0 - u_0 \|_1 \leq \tau_x.
\end{align}
According to the Lipschitz property of safety index $\phi$, we also have the following condition holds: 
\begin{align}\label{eq:phi_tau}
    \|\phi(x) - \phi(x_{\tau_0})\|_1 &\leq L_\phi \| x - x_{\tau_0}\|_1 = L_\phi \tau_x
\end{align}
which indicates that $\phi(x_{\tau_0}) \leq \phi(x)+L_\phi\tau_x$, hence:
\begin{align}\label{eq:phi_x_tau}
\max\{\phi(x_{\tau_0})-\eta, 0\} \leq \max\{\phi(x)-\eta  +L_\phi\tau_x, L_\phi\tau_x\}.
\end{align}

To bound $\sigma_f(x,u_0)$ in \eqref{eq:safe_control_discretization_proof_1}, we can follow a similar proof in Lemma \ref{lem:bound_posterior_variance_multidimensional} by replacing (i) $z_*$ with $(x, u_0)$, and (ii) $z_i$ with $(x_{\tau_0}, u_0)$. Therefore, we can replace $\| z_* - z_i \| \leq \tau$ in \eqref{eq:from_z_to_z_j_multidimensional}  with
\begin{align}
    \label{eq:from_z_to_x_u_0}
    \| (x, u_0) - (x_{\tau_0}, u_0)\|_1 = \| x - x_{\tau_0}\|_1 + \|u_0 - u_0\|_1 \leq \tau_x.
\end{align}
Since $(x_{\tau_0}, u_0)$ is a data sample from $\mathcal{D}_\tau$, then $\sigma_f(x_{\tau_0}, u_0) = 0$, a similar inequality as \eqref{eq:bound_posterior_variance_multidimensional_inproorf} can be established, indicating 
\begin{align}
    \label{eq:tileda_sigma_x_u_0}
    &\sigma_f(x, u_0) \\ \nonumber 
    \leq & \; \tilde{\sigma}_f \\ \nonumber 
    = & \; n_x \sqrt{2L_k \tau_x +  2|\mathcal{X}_{\tau_x}| L_k \tau_x \|K^{-1} \| \max_{w, w' \in \mathcal{W}}k(w, w')}.
\end{align}
Plugging \eqref{eq:norm_tau}, \eqref{eq:phi_x_tau} and $\eqref{eq:tileda_sigma_x_u_0}$ into \eqref{eq:safe_control_discretization_proof_1}, we have that 
\begin{align}
    \label{eq:safe_control_discretization_proof_2}
    &\mathbf{U}_f(x, u_0) \\ \nonumber 
    \leq & \; L_\phi L_f \tau_x + \max\{\phi(x_{\tau_0})-\eta, 0\} - A - B - C \\ \nonumber
    & + 2L_\phi \beta_f \sigma_f(x, u_0)\\ \nonumber
    \leq & \; \big(L_\phi L_f \tau_x - A \big) + \big(\max\{\phi(x_{\tau_0})-\eta, 0\} - B\big) \\ \nonumber  
    & + \big(2L_\phi \beta_f \tilde\sigma_f -C\big) \\ \nonumber
    = & \; \big(L_\phi L_f \tau_x - L_\phi L_f \tau_x \big) + \big(\max\{\phi(x_{\tau_0})-\eta, 0\} - L_\phi\tau_x \big) \\ \nonumber
    & \; + \big(2L_\phi \beta_f \tilde\sigma_f -  2L_\phi \beta_f \tilde{\sigma}_f\big) \\ \nonumber
    \leq & \; 0 + \big( \max\{\phi(x)-\eta  +L_\phi\tau_x, L_\phi\tau_x\} - L_\phi\tau_x \big) + 0 \\ \nonumber
    = & \; \max\{\phi(x)-\eta, 0\}~,
\end{align}
which directly yields the result.

\end{proof}

\subsection{Probabilistic Uniform Error Bound Leveraging Lipschitz Continuity of Unknown Function}
\label{sec: munich}
In this subsection, we can construct high-probability confidence intervals on the unknown system dynamics in \eqref{eq:dynamics fn} that fulfill the well-calibrated model using the Gaussian process leveraging the Lipschitz continuity of unknown function.
\begin{lemma}
[\citep{lederer2019uniform}, Theorem 3.3] Assume a zero mean Gaussian process defined through the continuous covariance kernel $k(·, ·)$ with Lipschitz constant $L_k$ on the set $\mathcal{Z}$. Furthermore, assume a continuous unknown function $g: \mathcal{Z} \to \mathbb{R}$ with Lipschitz constant $L_g$ and $N \in \mathbb{N}$ observations $y_i$. Then, the posterior mean function $\mu_g(.)$ and standard deviation $
\sigma_g(.)$ of a Gaussian process  conditioned on the training data $\{(z_i, y(z_i))\}_{i=1}^N$ are continuous with Lipschitz constant $L_{\mu_g}$ and modulus of continuity $\omega_{\sigma N}(.)$ on $\mathcal{Z}$ such that
\begin{align}
    &L_{\mu_g} \leq L_k\sqrt{N}\|(K + \sigma_{\text{noise}}^2 I_N)^{-1} y_N \| \\
    &\omega_{\sigma N}(\tau) \\ \nonumber
    &\leq \sqrt{2\tau L_k \big( 1 + N \| K + \sigma_{\text{noise}}^2 I_N)^{-1}\| \max_{z, z' \in\mathcal{Z}} k( z, z')\big)}.
\end{align}

Moreover, pick $\delta \in (0,1), \tau \in \mathbb{R}_+$ and set 
\begin{align}
    \beta(\tau) &= 2\log\big(\frac{M(\tau, \mathbb{Z})}{\delta}\big) \\
    \gamma(\tau) &= (L_{\mu_g} + L_g) \tau + \sqrt{\beta(\tau) }\omega_{\sigma N}(\tau).
\end{align}
where $M(\tau, \mathbb{Z})$ is the covering number defined in~\citep{lederer2019uniform}. Then, it holds that 
\begin{align}
    &P\big( |g(z) - \mu_g(z) | \leq \sqrt{\beta(\tau)}\sigma_g(z) + \gamma(\tau), \forall z \in \mathcal{Z} \big) \\ \nonumber 
    &\geq 1 - \delta.
\end{align}
\label{lem:uniform_error_bound_with_guarantee}
Note that $\tau$ can be chosen arbitrarily small such that the effect of the constant $\gamma(\tau)$ can always be reduced to an amount which is negligible compared to $\sqrt{\beta(\tau)}\sigma_N(x)$.
\end{lemma}

\subsection{Lipschitz Constant for Safety Index}
\label{appendix: Lipschitz constant for safety index}
\begin{lemma}
\label{lem: L phi}
Denote the safety index design parameterization $\{n, k\}$, and denote $L_{d_x}$ and $L_{\dot{d}_x}$ as the 1-norm Lipschitz constant of $d$ and $\dot{d}$ with respect to $x$. The 1-norm Lipschitz constant for safety index is $L_{\phi} = \max\{nd_{min}^{n-1}, k\}(L_{d_x} + L_{\dot{d}_x})$ when $0 < n < 1$, and $L_{\phi} = \max\{nd_{max}^{n-1}, k\}(L_{d_x} + L_{\dot{d}_x})$ when $ n \geq 1$.
\end{lemma}

\begin{proof}
To derive the 1-norm Lipschitz constant $L_\phi$ of $\phi$ with respect to $x$, we first derive the 1-norm Lipschitz constant $L_{\phi_{d, \dot{d}}}$ of $\phi$ with respect to $\{d, \dot{d}\}$. Since the safety index $\phi(d, \dot{d}) = \sigma + d^n_{min} - d^n - k \dot{d}$ is a continuous and differentiable function on a convex state space,
where $d \in [d_{min}, d_{max}]$ and $\dot d \in [\dot d_{min}, \dot d_{max}]$. For any $x, y \in \mathcal{D}$, we define 
$\Phi(c) = \phi((1 -c)x + cy)$. Since $\Phi(c)$ is a continuous function with respect to $c$, we have
\begin{align}
    \Phi(1) - \Phi(0) = \nabla_c\Phi(c)(1 - 0)
\end{align}
for some $c \in [0, 1]$ according to the mean value theorem~\citep{flett19582742}.

Since $\Phi(1) = \phi(y)$, and $\Phi(0) = \phi(x)$, the following equation holds 
\begin{align}
    \phi(y) - \phi(x) &= \nabla_c\Phi(c) \\ \nonumber
    & = \frac{\partial \phi((1 -c)x + cy)}{\partial ((1 -c)x + cy)} \frac{\partial ((1 -c)x + cy)}{\partial c} \\ \nonumber
    & = \nabla \phi(z) (y - x)
\end{align}
for some $z=(1 -c)x + cy)$ within the line segment between $x$ and $y$. According to Hölder's inequality~\citep{beckenbach1966holder}, the following condition holds: 
\begin{align}
    \| \phi(y) - \phi(x)\|_1 & \leq \| \nabla \phi(z)\|_{\infty}\| y -x\|_1 \\ \nonumber
    \frac{\| \phi(y) - \phi(x)\|_1}{\| y -x\|_1} &\leq \| \nabla \phi(z)\|_{\infty}
\end{align} 

Since $\nabla \phi = [-nd^{n-1}, -k]^T$, we have
\begin{align}
    \max \| \nabla \phi\|_\infty = \Big\| \begin{bmatrix} \max_d nd^{n-1} \\ k \end{bmatrix} \Big\|_\infty
\end{align}

Note that $\max_d nd^{n-1} = nd_{min}^{n-1}$ when $0 < n < 1$, and $\max_d nd^{n-1} = nd_{max}^{n-1}$ when $n \geq 1$. Therefore, by setting $L_{\phi_{d, \dot{d}}} = \|[nd_{min}^{n-1}, k]^T \|_{\infty} = \max\{nd_{min}^{n-1}, k\}$ when $0 < n < 1$, and $L_{\phi_{d, \dot{d}}} = \|[nd_{max}^{n-1}, k]^T \|_{\infty} = \max\{nd_{max}^{n-1}, k\}$ when $ n \geq 1$, the following condition holds: 
\begin{align}
    L_{\phi_{d, \dot{d}}} = \max \| \nabla \phi\|_\infty  \geq \|\nabla \phi(z)\|_{\infty} \geq \frac{\| \phi(y) - \phi(x)\|_1}{\| y -x\|_1}.
\end{align}
With the 1-norm Lipschitz constant $L_{\phi_{d, \dot{d}}}$ of $\phi$ with respect to $\{d, \dot{d}\}$, we can derive the 1-norm Lipschitz constant $L_{\phi}$ of $\phi$ with respect to $x$:
\begin{align}
\label{eq:lip_phi_x_derivation}
    & |\phi(x_1) - \phi(x_2)| \\ \nonumber 
     = &\; |\phi\big(d(x_1), \dot{d}(x_1)\big) - \phi\big(d(x_2), \dot{d}(x_2)\big)| \\ \nonumber
     \leq &\; L_{\phi_{d, \dot{d}}} |\big(d(x_1), \dot{d}(x_1)\big) - \big(d(x_2), \dot{d}(x_2)\big)| \\ \nonumber
     = &\; L_{\phi_{d, \dot{d}}} \big(|d(x_1)-d(x_2)| + |\dot{d}(x_1) - \dot{d}(x_2)|\big)\\ \nonumber
     \leq &\; L_{\phi_{d, \dot{d}}} L_{d_x} |x_1 - x_2| + L_\phi L_{\dot{d}_x} |x_1 - x_2| \\ \nonumber
    \leq &\; L_{\phi_{d, \dot{d}}} (L_{d_x} + L_{\dot{d}_x}) |x_1 - x_2|,
\end{align}
where denote $L_{d_x}$ and $L_{\dot{d}_x}$ as the 1-norm Lipschitz constant of $d$ and $\dot{d}$ with respect to $x$. As shown in \eqref{eq:lip_phi_x_derivation}, the lase inequality is essentially the definition of the 1-norm Lipschitz constant $L_\phi$ of $\phi$ with respect to $x$. Therefore we have:
\begin{align}
\label{eq:lip_phi_x}
    L_\phi = L_{\phi_{d, \dot{d}}} (L_{d_x} + L_{\dot{d}_x}).
\end{align}

\end{proof}

\subsection{Proof of Theorem \ref{theorem:main}}
\label{Proof of theorem:main}
\begin{proof}



Since $\mathbf{U}_f(f(x_\tau, u)) = \phi(\mu_f(x_\tau, u)) + L_\phi \beta_f \sigma_f(x_\tau, u)$,  the fundamental condition for the nonempty set of safe control is that  $\forall x_i \in \mathcal{X}_{\tau_x}$,
\begin{align}
\label{eq: fundamental theta}
    & \exists u \in \mathcal{U}, \text{ s.t. } \\ \nonumber
    & \phi_\theta(\mu_f(x_i, u)) + L_\phi \beta_f \sigma_f(x_i, u) \\ \nonumber 
    < & \max\{\phi_\theta(x_i)-\eta, 0\} - L_\phi L_f \tau_x - L_\phi \tau_x  - 2L_\phi \beta_f \tilde{\sigma}_f
\end{align}

where $\theta = \{\sigma, n, k\}$ is the tunable parameters of the safety index.

\paragraph{ \eqref{eq: fundamental theta}  is equivalent to \eqref{eq: lhs}:}

According to Lemma \ref{lem: L phi}, 
and the fact that $n = 1$ according to safety index design, we have
$L_\phi = \max\{1, k\}(L_{d_x} + L_{\dot{d}_x})$. Therefore, the LHS of \eqref{eq: fundamental theta} can be written as: 
\begin{align}
\label{eq: LHS of fundamental}
    &\phi_\theta(\mu_f(x_i, u)) + L_\phi \beta_f \sigma_f(x_i, u) \\ \nonumber
    = &\sigma + d_{min} - d\big(\mu_f(x_i, u)\big) - k\dot{d}\big(\mu_f(x_i, u)\big) \\ \nonumber
    & + \max\{1, k\}(L_{d_x} + L_{\dot{d}_x}) \beta \sigma_f(x_i, u)~,
\end{align}
where $d(\cdot)$ and $\dot{d}(\cdot)$ represent the mapping from state to $d$ and $\dot{d}$ respectively. Similarly, RHS of \eqref{eq: fundamental theta} can be written as: 
\begin{align}
\label{eq: rhs of fundamental}
    &\max\{\phi_\theta(x_i)-\eta, 0\} - L_\phi L_f \tau_x - L_\phi \tau_x - 2L_\phi \beta_f \tilde{\sigma}_f \\ \nonumber
    = & \max\{\sigma + d_{min} - d_i - k \dot d_i-\eta, 0\} \\ \nonumber
    &- \max\{1, k\}(L_{d_x} + L_{\dot{d}_x}) L_f \tau_x \\ \nonumber
    & - \max\{1, k\}(L_{d_x} + L_{\dot{d}_x}) \tau_x \\ \nonumber
    & - 2\max\{1, k\}(L_{d_x} + L_{\dot{d}_x})\beta_f \tilde{\sigma}_f, \\ \nonumber 
    \geq & \; \sigma + d_{min} - d_i - k \dot d_i-\eta \\ \nonumber
    &- \max\{1, k\}(L_{d_x} + L_{\dot{d}_x}) L_f \tau_x \\ \nonumber
    &- \max\{1, k\}(L_{d_x} + L_{\dot{d}_x}) \tau_x \\ \nonumber
    & - 2\max\{1, k\}(L_{d_x} + L_{\dot{d}_x})\beta_f \tilde{\sigma}_f,
\end{align}
where $d_i = d(x_i)$ and $\dot{d}_i = \dot{d}(x_i)$. 

Since $\max\{\sigma + d_{min} - d_i - k \dot d_i-\eta, 0\} \geq \sigma + d_{min} - d_i - k \dot d_i-\eta$, \eqref{eq: fundamental theta} can be proved if there exists $u$, such that RHS of \eqref{eq: rhs of fundamental} is larger than RHS of \eqref{eq: LHS of fundamental}, i.e. 
\begin{align}
\label{eq: lhs}
    &  d(\mu_f(x_i, u)) + k\dot{d}(\mu_f(x_i, u)) \\ \nonumber
    &- \max\{1, k\}(L_{d_x} + L_{\dot{d}_x}) \beta \sigma_f(x_i, u) \\ \nonumber
    & - d_i - k \dot d_i-\eta - \max\{1, k\}(L_{d_x} + L_{\dot{d}_x}) L_f \tau_x \\ \nonumber
    & - \max\{1, k\} (L_{d_x} + L_{\dot{d}_x})\tau_x \\ \nonumber
    & - 2\max\{1, k\}(L_{d_x} + L_{\dot{d}_x})\beta_f \tilde{\sigma}_f > 0.
\end{align}


\paragraph{\eqref{eq: lhs} is equivalent to \eqref{eq: concise lhs}:}

By choosing $u = u_{GP,i}$ from Gaussian process dataset \\
$\{\big( (x_{i}, u_{GP,i}) , f(x_{i}, u_{GP,i})\big)\}_{i=1}^{|\mathcal{X}_{\tau_x}|}$, we have the following three conditions hold: 
\begin{align}
    \sigma_f(x_i, u_{GP,i}) &= 0 \\
    d\big(\mu_f(x_i, u_{GP,i})\big) &= d_{GP,i},\\
    \dot{d}\big(\mu_f(x_i, u_{GP,i})\big) &= \dot{d}_{GP,i},
\end{align}
where $d_{GP,i} = d(f(x_i, u_{GP,i}))$ and $\dot{d}_{GP,i} = \dot{d}(f(x_i, u_{GP,i}))$.
Therefore, by selecting $u = u_{GP,i}$ and condition \eqref{eq: lhs} becomes: 
\begin{align}
\label{eq: concise lhs}
    & k ( \dot d_{GP,i} - \dot d_i ) - (d_i - d_{GP,i}) -\eta \\ \nonumber
    &- \max\{1, k\}(L_{d_x} + L_{\dot{d}_x})L_f \tau_x \\ \nonumber
    &- \max\{1, k\}(L_{d_x} + L_{\dot{d}_x}) \tau_x \\ \nonumber
    & - 2\max\{1, k\}(L_{d_x} + L_{\dot{d}_x})\beta_f n_x \\ \nonumber
    &\cdot \sqrt{2L_k \tau_x +  2|\mathcal{X}_{\tau_x}| L_k \tau_x \|K^{-1} \| \max_{w, w' \in \mathcal{W}}k(w, w')} > 0. 
\end{align}


So far, we have shown that it is sufficient to prove \eqref{eq: concise lhs} holds in order to show \eqref{eq: fundamental theta} holds. Next, we will show that \eqref{eq: concise lhs} holds, given the discretization step size and safety index design.

\paragraph{ \eqref{eq: the choice of tau} and \eqref{eq: the choice of safety index} ensure \eqref{eq: concise lhs} holds:}
Since the state discretization size is selected according to \eqref{eq: the choice of tau}, by rearranging \eqref{eq: the choice of tau}, the following condition holds:
\begin{align}
\label{feasible_condition_for_k_when_k>=1_(3)}
& \frac{\inf_{x}\sup_{u}\Delta\dot{d}(x,u)}{2} - (L_{d_x} + L_{\dot{d}_x}) \sqrt{\tau_x} - \\ \nonumber 
&(L_{d_x} + L_{\dot{d}_x}) L_f \sqrt{\tau_x} -\bigg(2\sqrt{2L_k}(L_{d_x} + L_{\dot{d}_x})\beta_f n_x \\ \nonumber 
& \cdot \sqrt{1 +  |\mathcal{X}_{\tau_x}| \|K^{-1} \| \max_{z, z' \in \mathcal{Z}}k(z, z')}\bigg) \sqrt{\tau_x} > 0.
\end{align}

Since the Gaussian process dataset satisfies
\begin{align}
  \forall x_i, \dot d_{GP,i} - \dot d_i > \frac{\inf_{x}\sup_{u}\Delta\dot{d}(x,u)}{2},
\end{align}

Hence,
\begin{align}
\label{feasible_condition_for_k_when_k>=1_(2)}
& \dot d_{GP,i} - \dot d_i - (L_{d_x} + L_{\dot{d}_x}) \sqrt{\tau_x} \\ \nonumber 
& - (L_{d_x} + L_{\dot{d}_x}) L_f \sqrt{\tau_x} \\ \nonumber 
& -\bigg(2\sqrt{2L_k}(L_{d_x} + L_{\dot{d}_x})\beta_f n_x \\ \nonumber
& \cdot \sqrt{1 +  |\mathcal{X}_{\tau_x}| \|K^{-1} \| \max_{w, w' \in \mathcal{W}}k(w, w')}\bigg) \sqrt{\tau_x} > 0.
\end{align}

Given the fact that $\tau_x \leq 1$ (i.e., $\tau_x<
\sqrt{\tau_x}$) according to \eqref{eq: the choice of tau}, \eqref{feasible_condition_for_k_when_k>=1_(2)} indicates:
\begin{align}
\label{feasible_condition_for_k_when_k>=1}
\Omega = & \; \dot d_{GP,i} - \dot d_i - (L_{d_x} + L_{\dot{d}_x})\tau_x - (L_{d_x} + L_{\dot{d}_x}) L_f \tau_x \\ \nonumber 
& -2(L_{d_x} + L_{\dot{d}_x})\beta_f n_x \\ \nonumber 
& \cdot \sqrt{2L_k \tau_x +  2|\mathcal{X}_{\tau_x}| L_k \tau_x \|K^{-1} \| \max_{w, w' \in \mathcal{W}}k(w, w')} \\ \nonumber 
> &\; 0.
\end{align}

According to \eqref{eq: the choice of safety index}, 
denoting the LHS of \eqref{feasible_condition_for_k_when_k>=1} as $\Omega$, we have
\begin{equation}
\label{eq: for a single tau k}
k > 1, \text{ and } k > \frac{\eta + d_i - d_{GP,i}}{\Omega}
\end{equation}

With \eqref{eq: for a single tau k}, the LHS of \eqref{eq: concise lhs} becomes:
\begin{align}
    & k \Omega -\eta - (d_i - d_{GP,i}) \\ \nonumber 
    > & \frac{\eta + d_i - d_{GP,i}}{\Omega}  \Omega -\eta - (d_i - d_{GP,i}) \\ \nonumber 
    = &0,
\end{align}

which indicates \eqref{eq: concise lhs} holds. 

Hence, \eqref{eq: concise lhs} indicates \eqref{eq: lhs} holds, which further indicates \eqref{eq: fundamental theta} holds for every $x_i \in \mathcal{X}_\tau$.

\end{proof}

\subsection{Infimum of Supremum Safety Status Change}
\label{sec: inf sup safety status}

It is worth noting that the system property $\inf_x\sup_u\Delta \dot d(x,u)$ is crucial for establishing the nonempty set of safe control theorem as indicated in \eqref{eq: the choice of tau} and \eqref{eq: the choice dataset}. In fact, if little prior knowledge of the unknown dynamics function $f(\cdot)$ is given, it is intractable to derive the specific value of $\inf_x\sup_u\Delta \dot d(x,u)$. However, under a mild assumption regarding the Lipschitz continuity of $\Delta \dot d(x,u)$,  
we can derive the lower bound of $\inf_x\sup_u\Delta \dot d(x,u)$, which is summarized in the following proposition.

\begin{theorem}
\label{theo: probabilistic infimum of supremum}
Under \Cref{asm:safe_control}.
Consider a state-space $\tilde{\tau}$-discretization
$\mathcal{X}_{\tilde{\tau}}$ defined in Definition \ref{definition:discretization}.

For each discretized state $x_i$, 
perform grid sampling by iteratively increasing sampling resolution to find a control ${u_{safe}}_i$, s.t. $\Delta \dot d(x_i, {u_{safe}}_i) > 0$. 

Denote $L_{\Delta\dot d}$ as the Lipschitz constant for function $\Delta \dot d(x, u): \mathcal{U} \rightarrow \mathbb{R}$. Then,
the system property $\inf_x \sup_u \Delta \dot d(x,u)$ is lower bounded by:
\begin{align}
\label{eq: the inf and sup of delta d}
\inf_x \sup_u \Delta \dot d(x,u) \geq \min  \{ \underline{\sup \Delta \dot d_1}, \underline{\sup \Delta \dot d_2} \cdots, \underline{\sup \Delta \dot{d}_{|\mathcal{X}_{\tilde{\tau}}|}
} \} - L_{\Delta \dot d} \tilde{\tau}
\end{align}
where $\underline{\sup \Delta \dot d_{i}} = \Delta \dot d(x_i, {u_{safe}}_i)$.
\end{theorem}

The proof for \Cref{theo: probabilistic infimum of supremum} is summarized in \Cref{proof for theo: probabilistic infimum of supremum}, where we use (i) grid sampling and (ii) Lipstchiz continuity to derive a lower bound for the maximum value of the unknown function (in our case $\Delta \dot{d}(\cdot)$).


\Cref{theo: probabilistic infimum of supremum} states that we first discretize the continuous state space $\mathcal{X}$ into discretized state space $\mathcal{X}_{\tilde{\tau}}$. Then, for each discretized state $x_i$, we use grid sampling by iteratively increasing the sampling resolution to find a control such that safety status change is positive, which results a lower bound of $\sup_u \Delta \dot d(x_{i}, u)$ for  each discretized state. In Lemma \ref{lem:finiteness}, we show the grid sampling procedure can be finished within finite many iterations for each discretized state.


Finally, with the Lipschitz constant of $\Delta \dot d (x, u)$, the lower bound of $\inf_x\sup_u\Delta \dot d(x,u)$ on the continuous state space $\mathcal{X}$ can be established according to \eqref{eq: the inf and sup of delta d}. Note that all the values required in \eqref{eq: the inf and sup of delta d} can be directly computed, where the Lipschitz constant $L_{\Delta\dot d}$ of an unknown function $\Delta \dot d(\cdot, \cdot)$ can be computed analytically via GP~\citep{lederer2019uniform}. 
Therefore, the lower bound of $\inf_x\sup_u\Delta \dot d(x,u)$ obtained with \Cref{theo: probabilistic infimum of supremum} can be directly used in \Cref{theo: nonempty set of safe control} for guaranteeing nonempty set of safe control in practical applications.

\subsection{Finite Sampling for Safety Status Change}
\label{sec: Probabilistic Infimum of Supremum for Gaussian Process}

\begin{lemma}[Existence]
Under \Cref{asm:safe_control}, then for each discretized state $x_i$, we can find a ${u_{safe}}_i$, s.t. $\Delta \dot d(x_i, {u_{safe}}_i) > 0$, with finite many iterations.
\label{lem:finiteness}
\end{lemma}

\begin{proof}

Under \Cref{asm:safe_control}, we have the infimum of the supremum of $\Delta \dot d$ can achieve positive, i.e., $\inf_x\sup_u\Delta \dot d(x,u) > 0$. Since $\Delta \dot d(x,u)$ is a Lipschitz continuous function, we have the existence of a non-trivial set of safe control $\mathcal{U}^S_i$
for each discretized state $x_i$, such that
\begin{align}
    \forall u \in \mathcal{U}^S_i, \Delta \dot d(x_i,u) \geq 0.
\end{align}

Hence, the following condition holds

\begin{equation}\label{eq:nontrivial}
    \exists \mathcal{Q} \subset \mathcal{U}^S_i, \exists \kappa >0, \text{\ s.t. \ }  s > \kappa~,
\end{equation}
where $\mathcal{Q}$ is a $n_u$-dimensional hypercube with the same length of $s$. Denote $\zeta_{[i]} = \max_{j,k}\| u_{[i]}^j -  u_{[i]}^k \|$, where $u_{[i]}$ denotes the $i$-th dimension of control $u$, and $u^j \in \mathcal{U}^S_i,  u^k \in \mathcal{U}^S_i$. 

By directly applying grid sampling in $\mathcal{U}^S_i$ with sample interval $s^*$ at each control dimension, such that $ s^* < s$. The maximum sampling time $T^a$ for finding a control ${u_{safe}}_i$ in \Cref{theo: probabilistic infimum of supremum} satisfies the following condition:
\begin{equation}\label{eq:max_sample}
    T^a < \prod\limits_{i=1}^{n_u} \lceil\frac{\zeta_{[i]}}{s^*}\rceil~,
\end{equation}
where $T^a$ is a finite number due to infimum condition of $s$ in \eqref{eq:nontrivial}. Then we have proved that for each $x_i \in \mathcal{X}_{\tilde{\tau}}$, we can find a control ${u_{safe}}_i$, s.t. $\Delta \dot d(x_i, {u_{safe}}_i) > 0$, with finite many iterations with finite iteration (i.e. finite sampling time).

\end{proof}

\subsection{Proof for Theorem \ref{theo: probabilistic infimum of supremum}}
\label{proof for theo: probabilistic infimum of supremum}
\begin{proof}
With discretized state space $\mathcal{X}_{\tilde{\tau}}$, the following inequality holds:
\begin{align}
\label{eq: max ineq}
    \forall x\in \mathcal{X}, \sup_{u \in \mathcal{U}} \Delta \dot d(x,u) \geq \Delta \dot d(x,u^{\text{near}})
\end{align}
where $u^{\text{near}} = \argmax_{u \in \mathcal{U}}\Delta \dot d(x^{\text{near}}, u)$, and $x^{\text{near}} \in \mathcal{X}_{\tilde{\tau}}$ such that $\| x - x^{\text{near}}\|_1 \leq \tilde{\tau}$. 
According to \eqref{eq: max ineq}, we have:
\begin{align}
    \forall x\in \mathcal{X}, &\sup_{u \in \mathcal{U}} \Delta \dot d(x,u) \\ \nonumber 
    &\geq \Delta \dot d(x^{\text{near}},u^{\text{near}}) + \Delta \dot d(x,u^{\text{near}}) \\ \nonumber 
    & \quad - \Delta \dot d(x^{\text{near}},u^{\text{near}}) \\ \nonumber 
    &\geq \Delta \dot d(x^{\text{near}},u^{\text{near}}) - L_{\Delta \dot d} \tilde{\tau} \\ \nonumber 
    &\geq \min \{\sup \Delta \dot d_{1}, \sup \Delta \dot d_{2}, \cdots, \sup \Delta \dot d_{|\mathcal{X}_{\tilde{\tau}}|}\} \\ \nonumber
    & \quad - L_{\Delta \dot d} \tilde{\tau}
\end{align}

where $\sup \Delta \dot d_i$ denotes the supremum $\Delta \dot d$ for $i$-th discretized state $x_i$. Therefore, the infimum of the supremum $\Delta \dot d$ is lower bounded by:
\begin{align}
\label{eq: before final conclusion}
    & \inf_x \sup_u \Delta \dot d(x,u) \\ \nonumber
    & \geq \min \{\sup \Delta \dot d_{1}, \sup \Delta \dot d_{2}, \cdots, \sup \Delta \dot d_{|\mathcal{X}_{\tilde{\tau}}|}\} - L_{\Delta \dot d} \tilde{\tau}.
\end{align}

According to Lemma \ref{lem:finiteness}, for each discretized state $x_i$, we can find a ${u_{safe}}_i$, s.t. $\Delta \dot d(x_i, {u_{safe}}_i) > 0$. Hence, the $\sup \Delta \dot d_i$ is bounded by
\begin{align}
\label{eq: sup lower bound}
    \sup \Delta \dot d_{i} > \Delta \dot d(x_i, {u_{safe}}_i) = \underline{\sup \Delta \dot d_{i}}.
\end{align}

By plugging \eqref{eq: sup lower bound} into \eqref{eq: before final conclusion}, 
\begin{align}
\label{eq: finally}
    & \inf_x \sup_u \Delta \dot d(x,u) \\ \nonumber
    & \geq \min  \{ \underline{\sup \Delta \dot d_1}, \underline{\sup \Delta \dot d_2} \cdots, \underline{\sup \Delta \dot{d}_{|\mathcal{X}_{\tilde{\tau}}|}
} \} - L_{\Delta \dot d} \tilde{\tau}.
\end{align}

\end{proof}


\subsection{Necessity of Grid Discretization}
\label{sec: necessity of grid}
\begin{lemma}
    \label{lem: grid necessity}
    The grid-based discretization of state space is necessary to guarantee nonempty set of safe control for all possible state.
\end{lemma}
\begin{proof}


For any $x\in\mathcal{X}$, the fundamental condition for nonempty set of safe control is 
\begin{align}
\label{eq: noempty nece}
    \exists u, \mathbf{U}(x, u) < \max \{\phi(x)-\eta, 0\}
\end{align}

We denote $(x_{\tau_c}, u_{\tau_c})$ as the closest state data sample from dataset, where $\|x - x_{\tau_c}\|_1 = \tau_c$. Follow the similar derivation from \eqref{eq:safe_control_discretization_proof_1} and \eqref{eq:safe_control_discretization_proof_2}, the following inequality holds with probability at least $1-\delta$:

\begin{equation}\label{eq:safe_control_discretization_proof_nece}
\begin{split}
    \mathbf{U}(x ,u_{\tau_c}) 
    \leq & \; L_\phi L_f \tau_c + \mathbf{U}_f(x_{\tau_c}, u_{\tau_c}) + 2L_\phi \beta_f \tilde\sigma_{max}, \\ \nonumber 
\end{split}
\end{equation}
where $\tilde\sigma_{max}$ denotes the maximum posterior variance. According to \eqref{eq:phi_tau}, $\max \{\phi(x_{\tau_c})-\eta, 0\} - L_\phi \tau_{c} < \max \{\phi(x)-\eta, 0\} $. Hence, to ensure \eqref{eq: noempty nece} holds, we should show the following condition holds
\begin{align}
     L_\phi L_f \tau_c + \mathbf{U}_f(x_{\tau_c}, u_{\tau_c}) + 2L_\phi \beta_f \tilde\sigma_{max} < \max \{\phi(x_{\tau_c})-\eta, 0\} - L_\phi \tau_{c}
\end{align}
yielding 
\begin{align}
\label{eq: nece 1}
    \mathbf{U}_f(x_{\tau_c}, u_{\tau_c}) < \max \{\phi(x_{\tau_c})-\eta, 0\} - L_\phi L_f \tau_c - 2L_\phi \beta_f \tilde\sigma_{max} - L_\phi \tau_{c}.
\end{align}

So far we have shown that, to guarantee nonempty set of safe control on any state, it is necessary to validate the condition \eqref{eq: nece 1} on its closest state data sample. Following the similar derivation in \Cref{Proof of theorem:main}, the equivalent condition for \eqref{eq: nece 1} is
\begin{align}
\label{eq: final nece}
    & k ( \dot d_{GP,i} - \dot d_i ) - (d_i - d_{GP,i}) -\eta \\ \nonumber
    &- \max\{1, k\}(L_{d_x} + L_{\dot{d}_x})L_f \tau_c \\ \nonumber
    &- \max\{1, k\}(L_{d_x} + L_{\dot{d}_x}) \tau_c \\ \nonumber
    & - 2\max\{1, k\}(L_{d_x} + L_{\dot{d}_x})\beta_f \tilde\sigma_{max} > 0,
\end{align}
which indicates the $\tau_c$ should be upper bounded. 

Therefore, to ensure nonempty set of safe control for any $x \in \mathcal{X}$, it is necessary to acquire a data sample $x_{\tau_c}$ such that $\tau_c = \|x - x_{\tau_c}\|_1$ and $\tau_c$ is upper bounded, indicating the necessity of grid-based discretization of state space.
\end{proof}

\section{Experiment Details}
\begin{figure}[t]
    \centering
    \subfigure[Point robot]{\includegraphics[width=30mm]{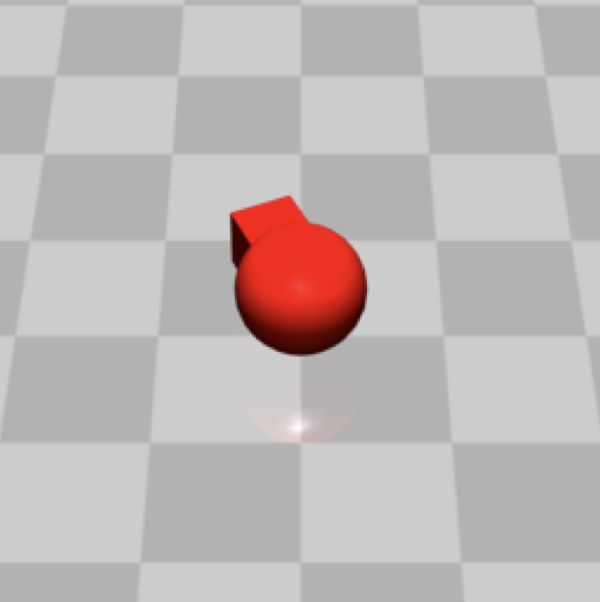}\label{fig:sg_point}}
    \hfill
    \subfigure[Hazard]{\includegraphics[width=30mm]{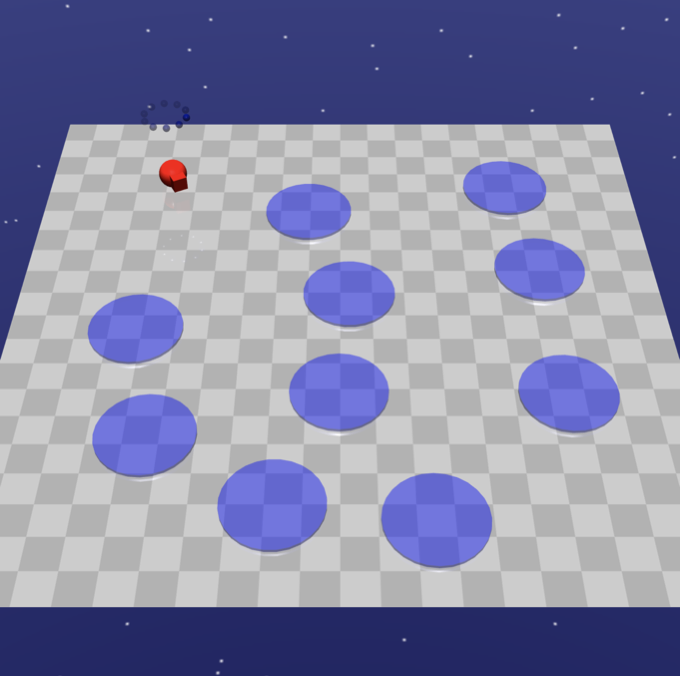}\label{fig:sg_hazard}}
    \hfill
    \subfigure[Goal task]{\includegraphics[width=30mm]{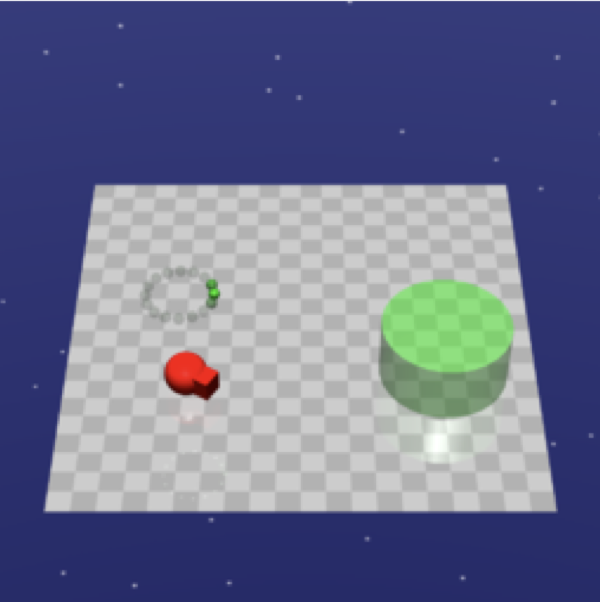}\label{fig:sg_goal}}
    \hfill
    \subfigure[Push task]{\includegraphics[width=30mm]{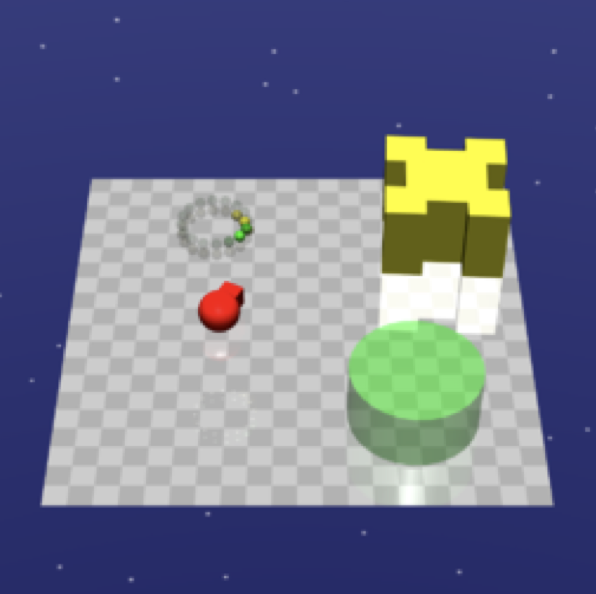}\label{fig:sg_push}}
    \centering
    \caption{Robots, obstacles and tasks in Safety Gym.}
    \label{fig:settings}
\end{figure}
\label{appendix: experiment details}

\subsection{Elements of Safety Gym}

Some key elements of the environment are listed below.
\begin{itemize}
    \item \texttt{Point Robot}: A 2D robot that can turn and move as shown in~\Cref{fig:sg_point}, with one actuator for turning and another for moving forward/backwards.
    \item \texttt{Hazard}: Dangerous areas to avoid as shown in~\Cref{fig:sg_hazard}. These are circles on the ground that are non-physical, and the robot is penalized for entering them.
    \item \texttt{Goal Task}: Move the robot to a series of goal positions as shown in \Cref{fig:sg_goal}. 
    \item \texttt{Push Task}: Thr robot needs to push the yellow box inside the green goal area as shown in \Cref{fig:sg_push}. 
\end{itemize}
\paragraph{State Space}
The state space is composed of various physical quantities from standard robot sensors (accelerometer, gyroscope, magnetometer, and velocimeter) and lidar (where each lidar sensor perceives objects of a single kind). The state spaces of all the test suites are summarized in \Cref{tab:state_space} (\Cref{appendix: experiment details}). In safety gym, the observation contains lidar information, which is difficult to predict. Thus, we only select accelerometer, gyroscope, magnetometer and velocimeter as state for GP prediction.

\paragraph{Control Space}
For all the experiments, the control space $\mathcal{U} \subset \mathbb{R}^2$. The first dimension $u_1 \in [-10, 10]$ is the control space of moving actuator, and second dimension $u_2 \in [-10, 10]$ is the control space of turning actuator. 

\subsection{Environment Settings}
\paragraph{Goal Task}
In the Goal task environments, the reward function is:
\begin{equation}\notag
\begin{split}
    & r(x_t) = d^{g}_{t-1} - d^{g}_{t} + \mathbbm{1}[d^g_t < R^g]~,\\
\end{split}
\end{equation}
where $d^g_t$ is the distance from the robot to its closest goal and $R^g$ is the size (radius) of the goal. When a goal is achieved, the goal location is randomly reset to someplace new while keeping the rest of the layout the same.

\paragraph{Push Task}
In the Push task environments, the reward function is:
\begin{equation}\notag
\begin{split}
    & r(x_t) = d^{r}_{t-1} - d^{r}_{t} + d^{b}_{t-1} - d^{b}_{t} + \mathbbm{1}[d^b_t < R^g]~,\\
\end{split}
\end{equation}
where $d^r$ and $d^b$ are the distance from the robot to its closest goal and the distance from the box to its closest goal, and $R^g$ is the size (radius) of the goal. The box size is $0.2$ for all the Push task environments. Like the goal task, a new goal location is drawn each time a goal is achieved. 


\paragraph{Hazard Constraint}
In the Hazard constraint environments, the cost function is:
\begin{equation}\notag
\begin{split}
    & c(x_t) = \max(0, R^h - d^h_t)~,\\
\end{split}
\end{equation}
where $d^h_t$ is the distance to the closest hazard and $R^h$ is the size (radius) of the hazard. We adopt two different sizes (0.15 and 0.30) of the hazard for evalution as shown in \Cref{fig:goal_and_push}.

\subsection{State Space Components}
The components of state space of the two environments (Goal-Hazard and Push-Hazard) in Safe Gym are shown in
\Cref{tab:state_space}.

\begin{table*}[t]
\vskip 0.15in
\begin{center}
\begin{tabular}{c|ccc}
\toprule
\textbf{State Space Option} &  Goal-Hazard & Push-Hazard\\
\hline
Accelerometer ($\mathbb{R}^3$) & \Checkmark & \Checkmark\\
Gyroscope ($\mathbb{R}^3$) & \Checkmark & \Checkmark\\
Magnetometer ($\mathbb{R}^3$) & \Checkmark & \Checkmark\\
Velocimeter ($\mathbb{R}^{3}$) & \Checkmark & \Checkmark\\
Goal Lidar ($\mathbb{R}^{16}$) & \Checkmark & \Checkmark\\
Hazard Lidar ($\mathbb{R}^{16}$) & \Checkmark & \Checkmark\\
Box Lidar ($\mathbb{R}^{16}$) & \XSolid & \Checkmark\\
\bottomrule
\end{tabular}
\end{center}
\caption{The state space components of different test suites environments.}
\label{tab:state_space}
\end{table*}


\subsection{Hyper-paramters}
The hyper-parameters used for evaluation in Safety Gym are reported in \Cref{tab:policy_setting}.

\begin{table}
\vskip 0.15in
\centering\footnotesize
\begin{tabular}{c|cccc}
\toprule
\textbf{Policy Parameter} & PPO & PPO-Lagrangian & CPO & PPO-SL \& PPO-UAISSA\\
\hline
Timesteps per iteration & 30000 & 30000 & 30000 & 30000 \\
Policy network hidden layers & (256, 256) & (256, 256) & (256, 256) & (256, 256) \\
Value network hidden layers & (256, 256) & (256, 256) & (256, 256) & (256, 256) \\
Policy learning rate & 0.0004 & 0.0004 & (N/A) & 0.0004 \\
Value learning rate & 0.001 & 0.001 & 0.001 & 0.001 \\
Target KL & 0.01 & 0.01 & 0.01 & 0.01 \\
Discounted factor $\gamma$ & 0.99 & 0.99 & 0.99 & 0.99 \\
Advantage discounted factor $\lambda$ & 0.97 & 0.97& 0.97 & 0.97 \\
PPO Clipping $\epsilon$ & 0.2 & 0.2  & (N/A) & 0.2 \\
TRPO Conjugate gradient damping & (N/A) & (N/A) & 0.1 & (N/A) \\
TRPO Backtracking steps & (N/A) & (N/A) & 10 & (N/A) \\
Cost limit & (N/A) & 0 & 0 & (N/A) \\
\bottomrule
\end{tabular}
\caption{Important hyper-parameters of PPO, PPO-Lagrangian, CPO, PPO-SL and PPO-UAISSA}
\label{tab:policy_setting}
\end{table}

\subsection{Ablation Study for Robot Arm}
\label{sec: alation study}
The ablation study on different discretization gap $\tau$. The results are shown in \Cref{fig:different_tau_robotarm}, where the gap between the upper bound of the safety index and the safety index decreases with smaller discretization gaps.

\begin{figure}[h]
    \centering
     \subfigure[$\frac{3}{8}\tau$, GP time = $1.35$ ms]{\includegraphics[width=0.45\columnwidth]{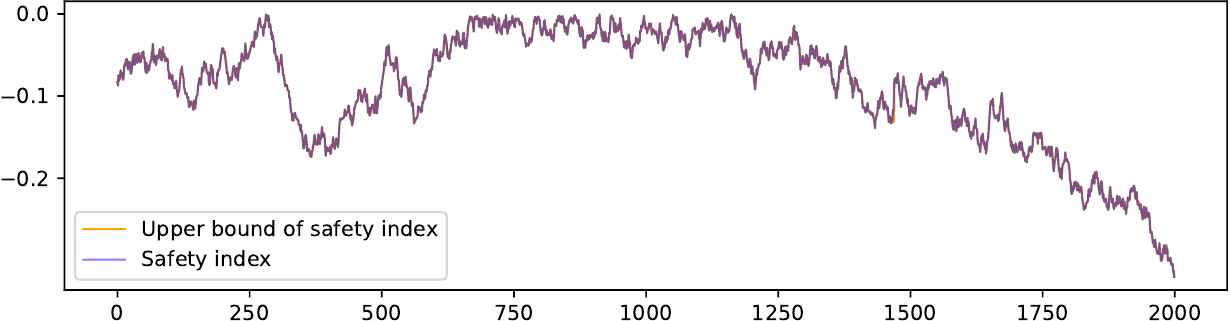}}
     \hfill
     \subfigure[$\frac{4}{8}\tau$, GP time = $0.527$ ms]{\includegraphics[width=0.45\columnwidth]{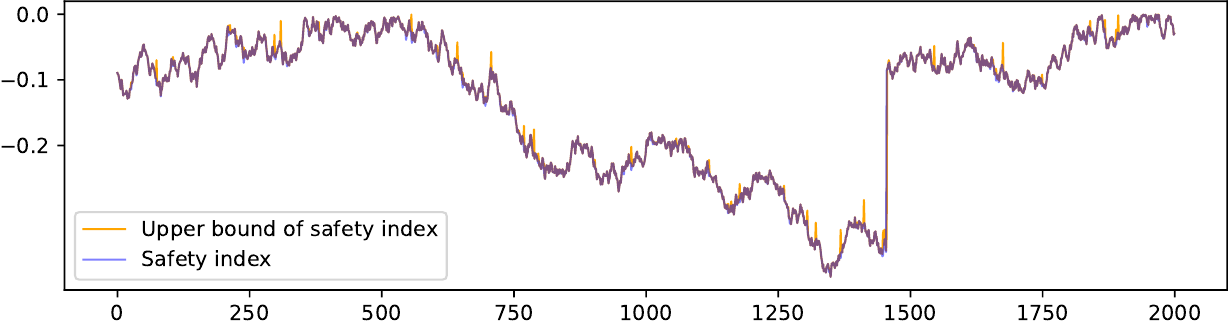}}
     \newline
     \subfigure[$\frac{7}{8}\tau$, GP time = $0.404$ ms]{\includegraphics[width=0.45\columnwidth]{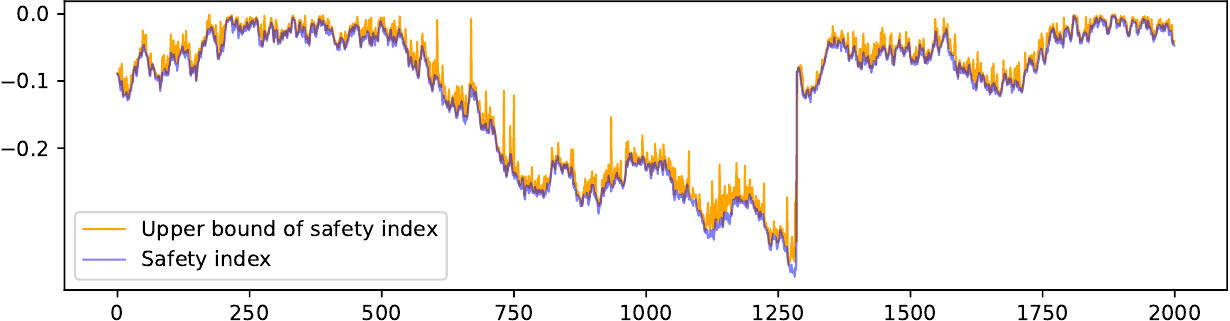}}
     \hfill    
     \subfigure[$\tau$, GP time = $0.367$ ms]{\includegraphics[width=0.45\columnwidth]{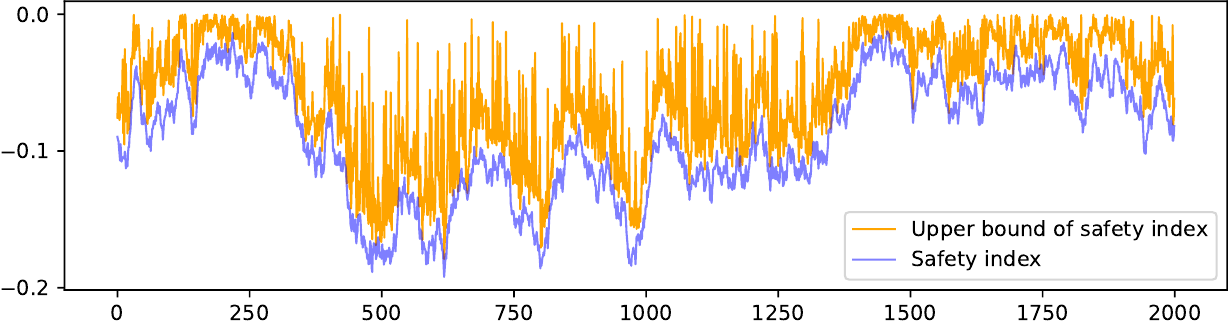}}
     \centering
    \caption{Evolutions of safety index and its upper bound with UAISSA with different discretization gap.}
    \label{fig:different_tau_robotarm}
\end{figure}

\subsection{Additional Experimental Details}
\label{sec:metrics}
\subsubsection{Offline Safety Index Synthesis and GP Dataset Construction}
In this subsection, we show the method to construct GP dynamics learning dataset and safety index synthesis in the high dimensional environments, i.e. SafetyGym. 

\paragraph{Alternative Upper Bound Posterior Variance for Safety Index Design}
In our experiments, we use deep GP to learn the system dynamics, which doesn't require the kernel functions. Hence, the kernel Lipschitz constant $L_k$ and posterior $K$ matrix is inaccessible, indicating \eqref{eq: the choice of tau} and \eqref{eq: the choice of safety index} cannot be evaluated. 

Note that the term $n_x \sqrt{2L_k \tau_x +  2|\mathcal{X}_{\tau_x}| L_k \tau_x \|K^{-1} \| \max_{w, w' \in \mathcal{W}}k(w, w')}$ is essentially the upper bound of posterior variance based on the collected data. To solve the aforementioned challenge, we leverage an alternative posterior variance upper bound theory~\citep{lederer2021uniformvariance}, which provides the worst case posterior variance upper bound with discretized input space
\begin{align}
    \tilde{\sigma}_f = \sigma_{rbf}^2 - \sigma_{rbf}^2 \exp\Big(-\frac{\tau_x^2}{2l^2}\Big)^2
\end{align}
where $\sigma_{rbf}$ and $l$ are the variance and lengthscale of RBF kernel, i.e. the prior for the GP.

Then according to \eqref{feasible_condition_for_k_when_k>=1}, the alternative $\tau_x$ condition for \eqref{eq: the choice of tau} becomes
\begin{align}
\label{eq: alternative tau}
& \; \inf_{x}\sup_{u}\Delta\dot{d} - (L_{d_x} + L_{\dot{d}_x})\tau_x - (L_{d_x} + L_{\dot{d}_x}) L_f \tau_x \\ \nonumber 
& -2(L_{d_x} + L_{\dot{d}_x})\beta_f n_x \bigg(\sigma_{rbf}^2 - \sigma_{rbf}^2 \exp\Big(-\frac{\tau_x^2}{2l^2}\Big)^2\bigg) > 0.
\end{align}

With the Lipsthiz constants and well calibrated model expression for $\beta_f$, we can select a proper $\tau_x$ satisfying \eqref{eq: alternative tau}. Subsequently, the safety index design in \eqref{eq: the choice of safety index} can also be evaluated.

\paragraph{Deep GP Dataset} We construct the dataset for deep GP by training a standard PPO agent and collecting the first 1e6 data samples (2.5 hours of running time) as the dataset. Based on this dataset, we derive the safety index parameters via evluating \eqref{eq: the choice of safety index} across  all the collected samples. Note that instead of grid sampling, our safety index design are based on data samples from small portion of the whole state space, which is less restrictive compared with \Cref{theorem:main}. And our experimental results show that this practical implementation empirically achieves near zero-violation performance.

\subsubsection{Metrics Comparison}
In this subsection, we report all the results of eight test suites by two metrics defined in Safety Gym \citep{ray2019benchmarking}:
\begin{itemize}
    \item The average episode return $J_r$.
    \item The average cost over the entirety of training $\rho_c$.
\end{itemize}
The average episode return $J_r$ and the average episodic sum of costs $M_c$ were obtained by averaging over the last five epochs of training to reduce noise. Cost rate $\rho_c$ was just taken from the final epoch. We report the results of these three metrics in \Cref{tab:metrics} normalized by PPO results.

\begin{table*}[t]
\setlength\tabcolsep{4pt}
\centering
  \subfigure[Goal-Hazard1-0.05]{%
    \begin{tabular}{c|cc}
    \toprule
    \textbf{Algorithm} & $\bar{J}_r$ & $\bar{\rho}_c$\\
    \hline
    PPO & 1.000 & 1.00000\\
    PPO-Lagrangian & \textbf{1.079} & 0.79350\\
    CPO & 1.022  & 0.43588\\
    PPO-SL & 0.997 & 1.03483\\
    PPO-UAISSA (Ours) & 0.939 & \textbf{0.00965}\\
    \bottomrule
    \end{tabular}}
  \hfill%
  \subfigure[Push-Hazard-0.15]{%
    \begin{tabular}{c|cc}
    \toprule
    \textbf{Algorithm} & $\bar{J}_r$ & $\bar{\rho}_c$\\
    \hline
    PPO & 1.000 & 1.00000\\
    PPO-Lagrangian & 1.008  & 0.35270\\
    CPO & 0.990  & 0.25780 \\
    PPO-SL & \textbf{1.011} & 1.07102\\
    PPO-UAISSA (Ours) & 0.884 & \textbf{0.00060}\\
    \bottomrule
    \end{tabular}}
  \newline%
  \subfigure[Goal-Hazard-0.30]{%
    \begin{tabular}{c|cc}
    \toprule
    \textbf{Algorithm} & $\bar{J}_r$ & $\bar{\rho}_c$\\
    \hline
    PPO & 1.000 & 1.00000\\
    PPO-Lagrangian & 1.008  & 0.35270\\
    CPO & 0.990  & 0.25780 \\
    PPO-SL & \textbf{1.011} & 1.07102\\
    PPO-UAISSA (Ours) & 0.884 & \textbf{0.00060}\\
    \bottomrule
    \end{tabular}} 
  \hfill%
  \subfigure[Push-Hazard-0.30]{%
    \begin{tabular}{c|cc}
    \toprule
    \textbf{Algorithm} & $\bar{J}_r$ & $\bar{\rho}_c$\\
    \hline
    PPO & 1.000 & 1.00000\\
    PPO-Lagrangian & 0.849  & 0.26067\\
    CPO & \textbf{1.192}  & 0.14833 \\
    PPO-SL & 0.757 & 0.66262\\
    PPO-UAISSA (Ours) & 0.689 & \textbf{0.00034}\\
    \bottomrule
    \end{tabular}}
   \centering
\caption{Normalized metrics obtained from the policies at the end of the training process, which is averaged over four test suits environments and five random seeds.}
\label{tab:metrics}
\end{table*}

\end{document}